\documentclass[letterpaper]{article} 
\usepackage{aaai2026}  
\usepackage{times}  
\usepackage{helvet}  
\usepackage{courier}  
\usepackage[hyphens]{url}  
\usepackage{graphicx} 
\usepackage{natbib}  
\usepackage{caption} 
\usepackage{algorithm}
\usepackage{algorithmic}
\usepackage{amsmath}
\usepackage{amssymb}
\usepackage{amsthm}
\usepackage{graphicx}
\usepackage[capitalize]{cleveref}
\usepackage{kotex}
\usepackage{arydshln}
\usepackage{colortbl}

\newtheorem{theorem}{Theorem}

\newtheorem{definition}[theorem]{Definition}
\newtheorem{lemma}[theorem]{Lemma}
\newtheorem{proposition}[theorem]{Proposition}
\newtheorem{corollary}[theorem]{Corollary}
\newtheorem{remark}[theorem]{Remark}

\DeclareMathOperator*{\argmax}{argmax}
\usepackage{cancel}
\usepackage{multirow}
\usepackage{booktabs}
\usepackage{url}
\usepackage{microtype}
\usepackage{xcolor}
\usepackage{subcaption}
\begin{document}
%
\title{Fair Model-Based Clustering}
\author{
    Jinwon Park\textsuperscript{\rm 1}
    ,
    Kunwoong Kim\textsuperscript{\rm 2}
    ,
    Jihu Lee\textsuperscript{\rm 2},
    Yongdai Kim\textsuperscript{\rm 2}
}
\affiliations{
    \textsuperscript{\rm 1}Graduate School of Data Science, Seoul National University\\
    \textsuperscript{\rm 2}Department of Statistics, Seoul National University\\
    $\{$jwpark127, kwkim.online, superstring1153, ydkim0903$\}$@gmail.com
}

\maketitle

\begin{abstract}
    The goal of fair clustering is to find clusters such that the proportion of sensitive attributes (e.g., gender, race, etc.) in each cluster is similar to that of the entire dataset.
    Various fair clustering algorithms have been proposed that modify standard K-means clustering to satisfy a given fairness constraint.
    A critical limitation of several existing fair clustering algorithms is that the number of parameters to be learned is proportional to the sample size because the cluster assignment of each datum should be optimized simultaneously with the cluster center, and thus scaling up the algorithms is difficult.
    In this paper, we propose a new fair clustering algorithm based on a finite mixture model, called Fair Model-based Clustering (FMC).
    A main advantage of FMC is that the number of learnable parameters is independent of the sample size and thus can be scaled up easily.
    In particular, mini-batch learning is possible to obtain clusters that are approximately fair. 
    Moreover, FMC can be applied to non-metric data (e.g., categorical data) as long as the likelihood is well-defined. 
    Theoretical and empirical justifications for the superiority of the proposed algorithm are provided.
\end{abstract}


\section{Introduction}

Artificial intelligence (AI) systems have been increasingly deployed as decision support tools in socially sensitive domains such as credit scoring, criminal risk assessment, and college admissions. 
AI models trained on real-world data may inherently encode biases present in the training data. As a result, these models can exhibit discriminatory behavior or amplify historical biases \citep{feldman2015certifying, barocas2016big, chouldechova2017fair, kleinberg2018algorithmic, mehrabi2021survey, zhou2021radfusion}.
Numerous studies have shown that, in the presence of such unfairness, these systems tend to favor majority groups, such as white men, while disadvantaging minority groups, such as black women.
This can lead to unfair treatment of socially sensitive groups \citep{dua2017uci, angwin2022machine}.
These findings underscore the need for fairness-aware learning algorithms to mitigate bias in automated decision-making systems.

Given these circumstances, it is crucial to prioritize fairness in automated decision-making systems to ensure that their development is aligned with the principles of social responsibility.
Consequently, a variety of algorithmic approaches have been introduced to reduce bias by promoting equitable treatment of individuals across demographic groups.
For instance, one common approach enforces fairness by requiring that the rates of favorable outcomes be approximately the same across groups defined by sensitive attributes such as race or gender \citep{calders2009building, pmlr-v130-gitiaux21a, zafar2017fairness, zeng2021fair, agarwal2018reductions, donini2018empirical, pmlr-v28-zemel13, 8622525, Quadrianto_2019_CVPR}.

In parallel with advances in supervised learning, the study of algorithmic fairness in unsupervised learning, especially clustering, has garnered increasing attention.
Clustering methods have long been used for various machine learning applications, including time series analysis \citep{LEE2024112434, 10.1145/2949741.2949758}, audio and language modeling \citep{brown-etal-1992-class, diez19_interspeech, butnaru2017image, zhang2023clusterllm}, recommendation systems \cite{10.4108/eai.17-7-2021.2312409, widiyaningtyas2021recommendation}, and image clustering \citep{le2013building, guo2020learning, mittal2022comprehensive}.

Fair clustering \cite{chierichetti2017fair} aims to ensure that the proportion of each protected group within every cluster is similar to that of the overall population.
To this end, a variety of algorithms have been proposed to minimize a given clustering objective (e.g., $K$-means clustering cost) while satisfying predefined fairness constraints \cite{bera2019fair, backurs2019scalable, li2020deep, esmaeili2021fair, ziko2021variational, zeng2023deep}.

Most fair clustering algorithms are based on the $K$-means (center) algorithm which searches for good cluster centers and a good assignment map (mapping each instance to one of the given cluster centers) under fairness constraints.
A critical limitation of such existing fair clustering algorithms is that the fairness of a given assignment map depends on the entire training data and thus learning the optimal fair assignment map is computationally demanding. 
In fact, the assignment of each data point must be learned along with the cluster centers under a fairness constraint, so the number of learnable parameters is proportional to the sample size. Moreover, a mini-batch learning algorithm, a typical tool to scale-up a given learning algorithm, would not be possible since the fairness of a given assignment map cannot be evaluated on a mini-batch.

Another limitation of existing fair clustering algorithms is that they are based on the $K$-means algorithm, which requires data to lie in a metric space, thus making it difficult to process non-metric data such as categorical data.

The aim of this study is to develop a new fair clustering algorithm that resolves the aforementioned limitations.
The proposed algorithm is based on model-based clustering, and we estimate the parameters under a fairness constraint.

Note that the existing methods are geometric and distance metric-based, yielding deterministic assignments, whereas model-based fair clustering assumes a probabilistic model for the data and uses likelihood-based estimation, providing probabilistic assignments.

An important advantage of the proposed fair clustering algorithm is that the number of learnable parameters is independent of the size of training data and thus can be applied to large-scale datasets easily.
Moreover, the development of a mini-batch algorithm to provide approximately fair clusters is possible with theoretical guarantees.
In addition, non-metric data can be processed without much difficulty as long as the likelihood can be well-specified.

The main contributions of this paper are summarized as follows.
\begin{itemize}
    \item We propose a novel fair clustering algorithm based on a probabilistic model that facilitates scalable mini-batch learning and dynamically recalculates fair assignments for large test datasets.
    \item Unlike most existing fair clustering methods, our proposed algorithms are built upon a probabilistic mixture model and naturally extend to categorical data.
    \item We establish a theoretical guarantee that the proposed model generalizes well to unseen data, and validate it through comprehensive empirical evaluations.
\end{itemize}

\section{Review of Existing Fair Clustering Algorithms}

There are three categories of existing methods for fair clustering.
(1) Pre-processing methods \citep{chierichetti2017fair, backurs2019scalable} use the concept of fairlets-small subsets of the data satisfying fairness-and then perform clustering on the space of fairlets.
(2) In-processing methods \citep{kleindessner2019guarantees, li2020deep, ziko2021variational, zeng2023deep, kim2025fairclusteringalignment} optimize both the assignments and cluster centers to find a good clustering among those satisfying a given fairness level.
(3) Post-processing methods \citep{bera2019fair, NEURIPS2020_a6d259bf} only build fair assignments for fixed cluster centers.
Here, we briefly review in-processing methods to explain the limitations of existing fair clustering algorithms.

Let $\mathcal{D} = \{x_1, \dots, x_N\}$  be the given training data to be clustered, and let $s_1,\ldots,s_N$ be sensitive attributes corresponding to $x_1,\ldots,x_N,$ where $s_i\in [M]$ for all $i \in [N].$
Here, $M$ is the number of sensitive groups. 
Let $\mathcal{D}^{(s)}=\{x_i\in \mathcal{D}: s_i=s\}$ for $s\in [M].$ 
We first focus on the case $M=2$ for the following discussion, and propose a way of extending the algorithm in the \textit{Modification of FMC to multinary sensitive attributes} section.

A typical clustering algorithm (without a fairness constraint) receives $\mathcal{D}$ and the number of clusters $K$ as inputs and yields the cluster centers $\mu_k, k\in [K]$, and the assignment map $\mathcal{A}:\mathcal{D} \rightarrow [K]$ such that $\mathcal{A}(x_i)$ is the cluster to which $x_i$ belongs.
The algorithm searches $\{\mu_k,\, k\in [K]\}$ and $\mathcal{A}$ that minimizes the clustering cost $L(\mathcal{D};K):=\sum_{i=1}^{N} \rho(x_i, \mu_{\mathcal{A}(x_i)})$ for a given metric $\rho.$ 

A fair clustering algorithm receives $\mathcal{D}^{(s)}$, $s\in [2]$ and $K$ as inputs and searches $\mu_k,\, k\in [K]$ and $\mathcal{A}$ that minimize $L(\mathcal{D};K)$ subject to the constraint that $\mathcal{A}$ is fair, i.e., $|\{i:\mathcal{A}(x_i)=k,\, s_i=1\}| / N_1 \approx |\{i:\mathcal{A}(x_i)=k,\, s_i=2\}| / N_2$ for all $k\in [K],$ where $N_s=|\mathcal{D}^{(s)}|.$
In certain cases, fair assignment maps do not exist.
To ensure existence, we consider a soft assignment map $\mathcal{A}_{soft}:\mathcal{D}\rightarrow \mathcal{S}^K,$ where $\mathcal{S}^K$ is the $(K-1)$-dimensional simplex \cite{kim2025fairclusteringalignment}.
Here, $\mathcal{A}_{soft}(x_i)_k$ is interpreted as the probability of $x_i$ belonging to cluster $k.$

A typical procedure for in-processing fair clustering algorithms is summarized as follows.
A fair clustering algorithm first finds the cluster centers using a fairness-agnostic clustering algorithm.
It then updates the (soft) assignment map while the cluster centers are fixed.
In turn, it updates the cluster centers while the assignment map is fixed, and it repeats these two updates iteratively until convergence. 

A problem in such fair clustering algorithms is that updating the assignment map is computationally demanding when the sample size is large, as the full optimization has cubic complexity: the number of learnable parameters is $N$, since $\mathcal{A}(x_i), i \in [N]$ must be updated simultaneously.
Moreover, a mini-batch algorithm—a typical way to scale-up a learning procedure—is not possible, since the fairness of a given assignment map depends on the entire dataset.

Another problem arises when assigning test data to clusters.
Assigning newly arrived data fairly after learning fair clusters is challenging, because the learned assignment map is defined only on the training data.
A practical solution would be to obtain the optimal assignment map for test data while the cluster centers are fixed.
This naive solution, however, is not applicable when test data arrive sequentially.

The aim of this paper is to propose a new fair clustering algorithm whose number of parameters is independent of the size of the training data and is thus easy to scale-up.
Moreover, it is possible to develop a mini-batch algorithm to provide an approximately fair clustering with theoretical guarantees.
In addition, the algorithm yields a parameterized assignment map, and thus assigning newly arrived unseen data is easy.

\section{Finite Mixture Models}

One renowned model-based clustering method is the finite mixture model, which assumes that
$x_1,\ldots,x_N\in\mathbb{R}^d$ are independent realizations of a $d$-dimensional random variable $X$ with density given by
\begin{equation}\label{eq:standard_mm}
    X \sim \mathbf{f}(\cdot; \Theta) := \sum_{k=1}^K \pi_k f(\cdot ; \theta_k),
\end{equation}
where \(K\) is the number of mixture components and \(f(\cdot ; \theta_k)\) is the density of a parametric distribution with parameter \(\theta_k\) for $k\in[K]$. The Gaussian distribution is commonly used for $f(\cdot;\theta)$ but other parametric distributions can also be used (e.g., the categorical (multinoulli) distribution for categorical data).
The mixture weight vector \(\boldsymbol{\pi} = (\pi_1,\dots,\pi_K)\) lies on the \((K-1)\)-dimensional simplex. 
From a clustering perspective, $K$ is the number of clusters; $f(\cdot;\theta_k)$ is the density of the $k$th component (cluster); and $\pi_k$ is the probability of belonging to the $k$th cluster.
The parameter \(\Theta = (\boldsymbol{\pi}, (\theta_1,\ldots,\theta_K))\) can be estimated by maximum likelihood (MLE), which maximizes the log-likelihood given by
$$\ell(\Theta \mid \mathcal{D})=\sum_{i=1}^N \log\left(\sum_{k=1}^K \pi_k f(x_i;\theta_k)\right).$$
Various algorithms have been proposed to compute the MLE in the finite mixture model.
Among these, the Expectation-Maximization (EM) algorithm \cite{titterington1985statistical, 6796884, 1b26ac91-1d67-38ea-b761-bcada2498f5c} and gradient-descent-based (GD) optimization algorithm \cite{10.1007/s11063-021-10599-3, 10.1007/s10107-019-01381-4} are most widely used.
\citet{xu1996convergence} showed that EM and GD for the finite mixture model are linked via a projection matrix. An advantage of EM is that the objective function is monotonically increasing without the need to set a learning rate.
In contrast, GD is more flexible and applicable to a broader range of differentiable objectives.

\subsection{Assignment Map for the Finite Mixture Model}

An equivalent formulation of the finite mixture model in \cref{eq:standard_mm} is given by:
\begin{eqnarray*}
    \footnotesize
     Z_1,\ldots,Z_N &\stackrel{\textup{i.i.d.}}\sim& {\rm Categorical}(\boldsymbol{\pi})\\
     X_i \mid Z_i &\sim& f(\cdot;\theta_{Z_i}), i\in [N].
\end{eqnarray*}
Here, from a clustering perspective, $Z_i$ is interpreted as the cluster index to which $x_i$ belongs.
Note that
\begin{equation}\label{eq:resp_fn}
   p(Z_i = k \mid X_i = x_i; \Theta) = \frac{\pi_k f (x_i; \theta_k) }{\sum_{l=1}^{K} \pi_l f(x_i; \theta_l)}.
\end{equation}
In this paper, we use (\ref{eq:resp_fn}) as the (soft) assignment map corresponding
to the finite mixture model with parameter $\Theta$ and denote it by $\psi_k(x_i;\Theta) = p(Z_i = k \mid X_i = x_i; \Theta)$ for $ k\in [K]$.

In practice, the hard assignment map $\psi_{hard}(x_i;\Theta)$ for any given $x_{i}$ can be determined
from the soft assignment map via $\psi_{hard}(x_i;\Theta)=\argmax_{k \in [K]} \{ \psi_{k}(x_i; \Theta) \}.$
In the case of the Gaussian mixture model, 
this hard assignment map is close to the optimal assignment map of the $K$-center algorithm
under regularity conditions.
See \cref{rmk:hard_soft} in the Appendix for a detailed discussion.

\section{Learning a Fair Finite Mixture Model}\label{sec:fairmm}

\subsection{Fairness Constraint for the Finite Mixture Model}

Let \(\mathcal{C} = \{C_1, \dots, C_K\}\) denote a given set of clusters of \(\mathcal{D}\) where \(C_k = \{x_i \in \mathcal{D} : Z_i = k\}\) is the \(k\)th cluster of \(\mathcal{D}\).
Let \(\mathcal{C}^1 = \{C^1_1,\dots,C^1_K\}\) and \(\mathcal{C}^2 = \{C^2_1, \dots, C^2_K\}\) denote the sets of clusters of \(\mathcal{D}^{(1)}\) and \(\mathcal{D}^{(2)}\), respectively, where \(C_k = C^1_k \cup C^2_k\).
The most commonly used fairness measure in clustering is \textit{Balance} \cite{chierichetti2017fair}, which is the minimum of the ratios of each sensitive attribute in every cluster.
The Balance of $\mathcal{C} = \{C_1,\ldots,C_K\}$ is defined as
\begin{equation}\label{def:Balance}
    \min_{k \in [K]} \min \left\{ |C^1_k| / |C^2_k|, |C^2_k| / |C^1_k| \right\}.
\end{equation}
Smaller Balance indicates less fair clusters, and vice versa.
\citet{bera2019fair} studied a clustering algorithm that finds optimal clusters under a Balance constraint by constructing fair assignments.

A similar but numerically easier measure for fairness than Balance is a difference-based rather than ratio-based measure.
That is, we define \textit{Gap} as
\begin{equation}\label{fair_level_origin}
    \max_{k\in [K]} \left | \frac{|C^1_k|}{N_1} - \frac{|C^2_k|}{N_2} \right|.
\end{equation}
A larger Gap means less fairness of clusters and vice versa.
\citet{kim2025fairclusteringalignment} considered the additive version of Gap defined by
\begin{equation}\label{fair_level_additive}
    \sum_{k\in [K]} \left | \frac{|C^1_k|}{N_1} - \frac{|C^2_k|}{N_2} \right|,
\end{equation}
and developed a fair clustering algorithm.

A natural extension of Gap for the assignment map $\psi_k(\cdot;\Theta),k\in [K]$, is
\begin{equation} \label{eq:GAP}
    \footnotesize
    \Delta(\Theta) := \max_{k\in [K]} \left\vert \frac{\sum_{x_{i} \in \mathcal{D}^{(1)}} \psi_{k}\left(x_i ; \Theta \right)}{N_{1}} - \frac{\sum_{x_j \in \mathcal{D}^{(2)}} \psi_{k} \left(x_j; \Theta \right)}{N_{2}} \right\vert.
\end{equation}
\normalsize
In this paper, we propose to estimate $\Theta$ under the constraint $\Delta(\Theta) \le \varepsilon$
for a prespecified fairness level $\varepsilon \ge 0$ (the smaller $\varepsilon,$ the fairer).
That is, the objective of our proposed fair clustering is:
\begin{equation}\label{obj const}
    \max_{\Theta} \ell(\Theta \mid \mathcal{D}) \text{ subject to } \Delta(\Theta) \le \varepsilon.
\end{equation}
Given a Lagrange multiplier \(\lambda \ge 0\), we can alternatively maximize
\(\ell(\Theta \mid \mathcal{D}; \lambda) := \ell(\Theta \mid \mathcal{D}) - \lambda \Delta(\Theta)\).
We abbreviate \(\Delta(\Theta)\) as \(\Delta\) where appropriate.

For the algorithms presented below, we identify the maximizing index \(k\) in \cref{eq:GAP} and then evaluate \(\Delta\) using that \(k\).

We use Gap instead of Balance for numerical stability when applying gradient-descent optimization.
Regarding the relationship between Balance and Gap, 
we note that:
(i) \cite{kim2025fairclusteringalignment} theoretically showed that $\vert \textup{Balance} - \min(n_{0}/n_{1}, n_{1}/n_{0}) \vert \le \textup{Additive Gap in \cref{fair_level_additive}}$ (where $\min(n_{0}/n_{1}, n_{1}/n_{0})$ is the balance of perfectly fair clustering),
and (ii) our numerical experiments show that smaller $\Delta$ corresponds to a larger Balance (\cref{fig:delta_balance} in the Appendix).

\subsection{Learning Algorithms: FMC-GD and FMC-EM}
To optimize the objective in \cref{obj const}, we employ gradient-descent (GD) and expectation–maximization (EM) algorithms, with pseudocode in \cref{alg:fairGD,alg:fairEM}. We hereafter refer to these as FMC-GD (Fair Model-based Clustering with Gradient Descent) and FMC-EM (Fair Model-based Clustering with Expectation–Maximization), respectively.
\begin{algorithm}[h]
\caption{FMC-GD (Gradient-Descent)\\
In practice, we set \((T, \gamma) = (10000, 10^{-3})\)}\label{alg:fairGD}
    \begin{algorithmic}[0]
        \STATE \textbf{Input}: Dataset \(\mathcal{D} = \{x_1,\dots,x_N\}\), Number of clusters \(K\), Lagrange multiplier \(\lambda\), Maximum number of iterations \(T\) and Learning rate \(\gamma\).
        \STATE \textbf{Initialize}: \(\Theta^{[0]} = (\boldsymbol{\eta}^{[0]}, \boldsymbol{\theta}^{[0]})\)        
        \WHILE{$\Theta$ has not converged and $t < T$}
            \STATE 
            $\boldsymbol{\theta}^{[t+1]} \gets \boldsymbol{\theta}^{[t]} - \gamma \frac{\partial}{\partial \boldsymbol{\theta}} \left( - \ell(\Theta^{[t]} \mid X) + \lambda \Delta(\Theta^{[t]}) \right)$
            
            \STATE 
            $\boldsymbol{\eta}^{[t+1]} \gets \boldsymbol{\eta}^{[t]} - \gamma \frac{\partial}{\partial \boldsymbol{\eta}} \left( - \ell(\Theta^{[t]} \mid X) + \lambda \Delta(\Theta^{[t]}) \right)$
            
            \STATE \(\boldsymbol{\pi}^{[t+1]} \gets \text{softmax} (\boldsymbol{\eta}^{[t+1]}) \)
            \STATE $t \gets t+1$
        \ENDWHILE
    \end{algorithmic}
\end{algorithm}
\subsubsection{(1) FMC-GD}

FMC-GD maximizes $\ell(\Theta \mid \mathcal{D}; \lambda)$ using gradient-descent-based optimization.
To improve numerical stability of the gradients, we reparameterize the mixture weights as $\boldsymbol{\pi}=\operatorname{softmax}(\boldsymbol{\eta})$.
The gradient formulas used in \cref{alg:fairGD} are provided in the Appendix, in the section \textit{Derivation of the Gradients in Algorithm 1}.

\subsubsection{(2) FMC-EM}

For FMC-EM, we consider the complete-data $Y = (\mathcal{D}, Z),$ where $Z=(Z_1,\ldots,Z_N)^\top,$
along with the complete-data log-likelihood $\ell_{comp}(\Theta \mid Y)$:
\begin{equation}
    \footnotesize
    \ell_{comp}(\Theta \mid Y) = \sum_{i=1}^N \sum_{k=1}^K \mathbb{I}(Z_{i} = k) \left\{ \log \pi_{k} + \log f(x_i; \theta_k) \right\}.
\end{equation}
As in the standard EM algorithm, FMC-EM iterates the E-step and M-step.
Let $\Theta^{[t]}$ be the updated parameter at the $t$th iteration.
Given $\mathcal{D}$ and $\Theta^{[t]},$ the $Q$ function (computed in the E-step and maximized in the M-step) is
\begin{equation}
    \footnotesize
    Q(\Theta \mid \Theta^{[t]}) = \mathbb{E}_{Z \mid \mathcal{D}; \Theta^{[t]}} \ell_{comp}(\Theta \mid Y).
\end{equation}
See \cref{eq:emm_q_function} in the Appendix for a detailed derivation of $Q(\Theta \mid \Theta^{[t]}).$
For FMC-EM, the \(Q\) function is modified to
\begin{equation}\label{eq:q_fair}
    \footnotesize
    Q_{fair}(\Theta \mid \Theta^{[t]}; \lambda) = \mathbb{E}_{Z \mid \mathcal{D}, \Theta^{[t]}} \left( \ell_{comp}(\Theta \mid Y) - \lambda \Delta(\Theta) \right).
\end{equation}
We compute $Q_{fair}$ in the E-step and update $\Theta$ to maximize $Q_{fair}$ in the M-step.
Since maximizing $Q_{fair}$ admits no closed-form solution, we instead apply gradient-based optimization to ensure an increase at each iteration.
This procedure is a special case of the Generalized EM (GEM) algorithm \cite{dempster1977maximum}, which requires only that each update of $\Theta$ increase the $Q$ function. 
To sum up, given $\Theta^{[t]}$ at the \(t\)th iteration, we update the parameters \(\Theta^{[t+1]}\) via gradient-descent-based optimization, choosing the learning rate to satisfy
\begin{equation}
    \footnotesize
    Q_{fair} (\Theta^{[t+1]} \mid \Theta^{[t]}; \lambda) \ge Q_{fair} (\Theta^{[t]} \mid \Theta^{[t]} ; \lambda).
\end{equation}
See \cref{fig:variations} in the Appendix for the empirical convergence of FMC-EM.

\begin{algorithm}[h]
    \caption{FMC-EM (Expectation-Maximization)\\
    In practice, we set \((T, R, \gamma) = (200, 10, 10^{-2})\)}\label{alg:fairEM}
    \begin{algorithmic}
        \STATE \textbf{Input:} Dataset \(\mathcal{D} = \{x_1,\dots,x_N\}\), Number of clusters \(K\), Lagrange multiplier \(\lambda\), Maximum numbers of iterations \(T, R\) and Learning rate \(\gamma\).
        \STATE \textbf{Initialize:} \(\Theta^{[0]} = (\boldsymbol{\eta}^{[0]}, \boldsymbol{\theta}^{[0]})\)
        \WHILE{$Q_{fair}$ has not converged and $t < T$}
        \STATE Compute $Q_{fair}(\Theta \mid \Theta^{[t]}; \lambda)$ in \cref{eq:q_fair}
            \WHILE{$r < R$}
                \STATE $\boldsymbol{\theta}_{(r+1)}^{[t]} \gets \boldsymbol{\theta}_{(r)}^{[t]} - \gamma \frac{\partial}{\partial \boldsymbol{\theta}} \left( - Q_{fair}(\Theta ; \Theta^{[t]}, \lambda) \right)$
                
                \STATE $\boldsymbol{\eta}_{(r+1)}^{[t]} \gets \boldsymbol{\eta}_{(r)}^{[t]} - \gamma \frac{\partial}{\partial \boldsymbol{\eta}} \left( - Q_{fair}(\Theta ; \Theta^{[t]}, \lambda) \right)$
                
                \STATE \(\boldsymbol{\pi}_{(r+1)}^{[t]} \gets \text{softmax } (\boldsymbol{\eta}_{(r+1)}^{[t]}) \)
                \STATE $r \gets r+1$
            \ENDWHILE
            \STATE \(\boldsymbol{\theta}^{[t+1]} \gets \boldsymbol{\theta}^{[t]}_{(R)}\)
            \STATE \(\boldsymbol{\pi}^{[t+1]} \gets \boldsymbol{\pi}^{[t]}_{(R)}\)
            \STATE $t \gets t+1$
        \ENDWHILE
    \end{algorithmic}
\end{algorithm}

\subsection{Mini-Batch Learning with Sub-Sampled $\Delta$}

When the size of the dataset is large, reducing computational cost becomes a critical concern, and the use of mini-batches can be a compelling approach.
The objective function of FMC consists of two terms: the log-likelihood and the fairness penalty term $\Delta$.
However, note that a mini-batch algorithm can be applied to the log-likelihood term but cannot be easily applied to $\Delta$.
Hence, we need a technique to reduce the computational cost of computing the gradient of $\Delta$.

For this purpose, we use the fact that $\Delta$ can be closely approximated by its value on sub-sampled data.
Let $\mathcal{D}_{n}$ be a dataset of size $n$ obtained by randomly sampling $n$ data points from $\mathcal{D}$,
and let $\Delta(\Theta;\mathcal{D}_n)$ be the $\Delta$ computed on $\mathcal{D}_n.$
Under regularity conditions, it can be shown that fair clusters under the sub-sampled $\Delta$ constraint are approximately fair in terms of the (population) $\Delta$ constraint.
For simplicity, we focus on the Gaussian mixture model: let $f(x;\theta)$ be a Gaussian distribution with mean vector $\mu$ and covariance matrix $\Sigma$, and thus $\Theta=(\boldsymbol{\pi},((\mu_1,\Sigma_1),\ldots, (\mu_K,\Sigma_K))).$
For given positive constants $\xi,\zeta$ and $\nu,$ let
 $\Phi_{\xi,\zeta,\nu}=\{\Theta: \xi<\min_k \{\pi_k\} \le \max_k \{\pi_k\} <1-\xi,
\max_k \|\mu_k\| \le \nu, \zeta< \min_k \lambda_{\min} (\Sigma_k) \le\max_k \lambda_{\max}(\Sigma_k) < 1/\zeta\},$
where $\lambda_{\min}$ and $\lambda_{\max}$ are the smallest and largest eigenvalues, respectively.
The following proposition gives a bound on the error between $\Delta(\Theta;\mathcal{D}_n)$ and $\Delta(\Theta)$.
To avoid technical difficulties, we assume that $\mathcal{D}_{n}$ is obtained by sampling with replacement from $\mathcal{D}.$
The proof is given in the Appendix.

\begin{proposition}\label{prop:generalization}
    Let $n_s=|\{x_i\in \mathcal{D}_n: s_i=s\}|$ for $s\in \{1, 2\}.$
    Then, with probability at least $1-\delta,$ we have 
    \begin{equation}
        \footnotesize
        \sup_{\Theta\in \Phi_{\xi,\zeta,\nu}} |\Delta(\Theta)-\Delta(\Theta;\mathcal{D}_n)|
        \le 4C\sqrt{\frac{d}{n'}}+8\sqrt{\frac{2\log(8/\delta)}{n'}}
    \end{equation}
    for some constant $C=C(\xi,\zeta,\nu,x_{\max}),$ where $x_{\max}=\max_{x\in \mathcal{D}} \|x\|_2$ and
    $n'=\min\{n_1, n_2\}.$
\end{proposition}

Motivated by Proposition \ref{prop:generalization}, we propose to estimate $\Theta$ by maximizing

$$\ell(\Theta \mid \mathcal{D}) - \lambda \Delta(\Theta; \mathcal{D}_n),$$
where we apply mini-batch learning to the term $\ell(\Theta \mid \mathcal{D})$ while $\mathcal{D}_n$ is fixed.
In FMC-GD, we take gradient steps for $\ell(\Theta \mid \mathcal{D})$ on a given mini-batch, whereas in FMC-EM, we compute $\ell_{comp}(\Theta \mid Y)$ in the E-step only on a given mini-batch.

See \cref{alg:fairEM_minibatch} in the Appendix for details of the mini-batch learning algorithm with the sub-sampled $\Delta$ constraint.
Moreover, \citet{karimi:hal-02334485} studied convergence properties of the mini-batch EM algorithm, and we show a similar result for the mini-batch GEM in \cref{cor:converge} in the Appendix.

We emphasize that neither mini-batch learning nor a sub-sampled fairness constraint is possible in existing fair clustering algorithms, since the assignment map for the entire dataset must be learned simultaneously.
The key innovation of the proposed fair finite mixture model is to use the parametric model in \cref{eq:resp_fn} as the (soft) assignment map.
Note that learning the parameters is computationally easier than directly learning the assignment map for the entire dataset.

\subsection{Sub-Sample Learning}

Another advantage of FMC is that the functional form of the assignment map is available after training, which allows us to fairly assign newly arrived data (which are not available during the learning of $\Theta$) to clusters.
That is, we assign an unseen datum $x$ to cluster $k$ with probability $\psi_k(x;\hat{\Theta}),$ where $\hat{\Theta}$ is the estimated parameter vector.
However, this fair post-assignment is not possible for existing fair clustering algorithms, since the assignment map is defined only for the training data.

The fair post-assignment also enables learning the fair mixture model on randomly selected sub-samples when the entire dataset is too large.
Specifically, we estimate $\Theta$ on randomly selected sub-samples $\mathcal{D}_n$ of size $n$ from $\mathcal{D},$ and then assign the data in $\mathcal{D}\setminus\mathcal{D}_n$ using the estimated assignment map $\{\psi_k(x;\hat{\Theta})\}_{k \in [K]}$.
Proposition \ref{prop:generalization} guarantees that the clusters obtained in this way are approximately fair.
In the \textit{Numerical experiments} section, we empirically show that the reduction in performance, as measured by clustering cost and $\Delta$, due to sub-sampling is minimal, unless $n$ is too small.

\section{Numerical Experiments}\label{section_experiments}

In our numerical experiments, we evaluate FMC by analyzing several real-world datasets:
(i) on three moderate-scale real datasets, we show that FMC is competitive with baseline methods; and
(ii) on a large-scale dataset consisting of millions of instances, FMC outperforms the baselines while requiring significantly less computational time.
In addition, we investigate mini-batch learning and sub-sample learning of FMC, and show that the two algorithms perform similarly when the sub-sample size is not too small (e.g., 5\%).

We additionally conduct parameter impact studies on (i) the number of clusters $K$ and (ii) the choice of covariance structure when $\mathbf{f}$ is a Gaussian mixture.

\begin{figure*}[h]
    \centering
    \includegraphics[width=0.33\linewidth]{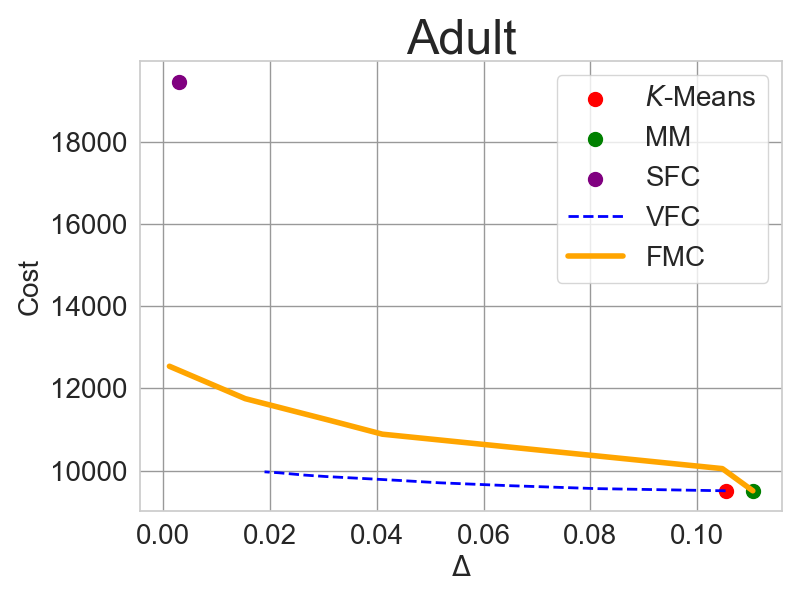}
    \includegraphics[width=0.33\linewidth]{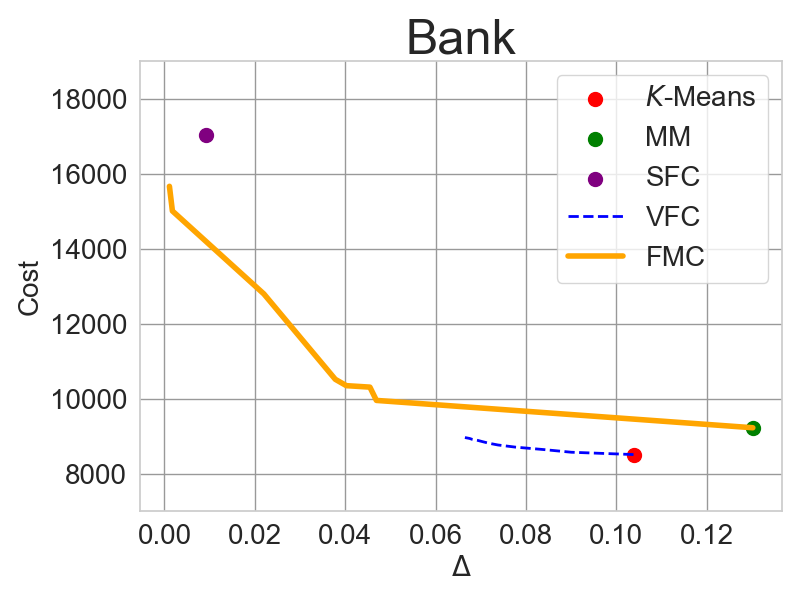}
    \includegraphics[width=0.33\linewidth]{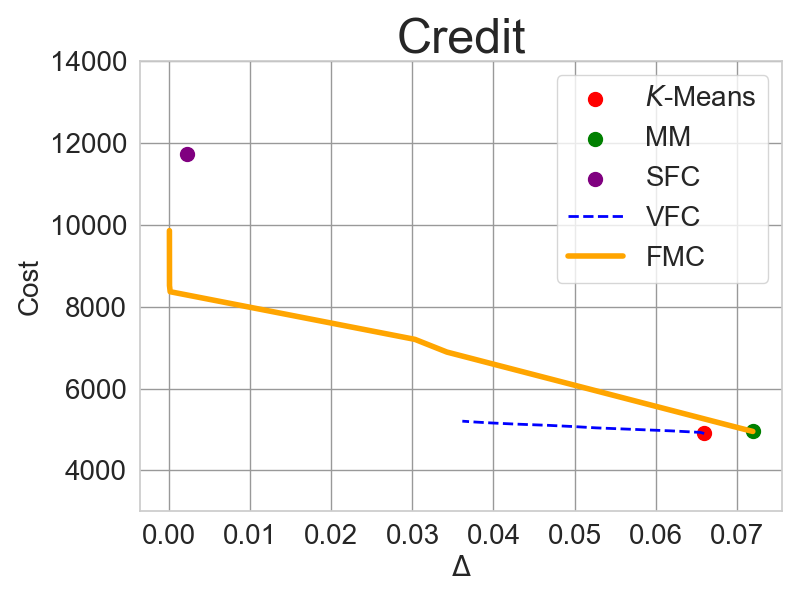}
    \normalsize
    \caption{Pareto front lines between $\Delta$ and Cost on Adult, Bank, and Credit datasets.
    See \cref{fig:three_L2_pareto_Balance} for the lines between Balance and Cost.
    See \cref{fig:three_noL2_pareto} for the similar results without $L_2$ normalization.
    }
    \label{fig:three_L2_pareto_delta}
\end{figure*}

\subsection{Experimental Settings}
\paragraph{Datasets}
We use four real-world datasets: Adult, Bank, Credit, and Census datasets, which are all available in the UCI Machine Learning Repository\footnote{\url{https://archive.ics.uci.edu/datasets}}.
Brief descriptions of the datasets are given below, with details in the Appendix.
\begin{itemize}
    \item {Adult dataset}: UCI Adult (1994 US Census), 32,561 samples; 5 continuous and 8 categorical features; sensitive attribute: gender (male 66.9\%, female 33.1\%).
    \item {Bank dataset}: UCI Bank Marketing, 41,008 samples; 6 continuous and 10 categorical features; sensitive attribute: marital status (married 60.6\%, unmarried 39.4\%).
    \item {Credit dataset}: UCI Default of Credit Card Clients, 30,000 samples; 4 continuous and 8 categorical features; sensitive attribute: education, aggregated into two groups (46.8\% vs. 53.2\%) or three groups (46.8\%, 35.3\%, 17.9\%).
    \item {Census dataset}: UCI Census (1990 US Census), 2,458,285 samples; 25 features; sensitive attribute: gender (male 51.53\%, female 48.47\%).
\end{itemize}

We analyze only the continuous variables in the main analysis, since we focus on the Gaussian mixture model.
Subsequently, we include categorical data to validate the applicability of FMC to categorical variables.
We standardize the continuous variables to have zero mean and unit variance.
Then, we apply $L_2$ normalization to the standardized data (i.e., so that each datum has unit $L_2$ norm),
because VFC \cite{ziko2021variational}, one of the baseline methods, often fails to converge without $L_2$ normalization.
Numerical results without $L_2$ normalization are given in the Appendix.

\paragraph{Algorithms}
For continuous data, we use the Gaussian mixture model (GMM) for $\mathbf{f}$ with an isotropic covariance matrix $\sigma^{2} \mathbb{I}_d.$ That is, the complete-data log-likelihood of the GMM is derived as 
\begin{equation}
\footnotesize
\begin{aligned}
    \footnotesize
    &\ell_{comp}(\Theta \mid Y) =\\
    &\sum_{i=1}^{N} \sum_{k=1}^{K} \mathbb{I}(Z_{i} = k) \left( \log \pi_{k} - \frac{d \log (2 \pi \sigma^{2})}{2} - \frac{\Vert x_{i} - \mu_{k} \Vert_{2}^{2}}{2 \sigma^{2}} \right),
\end{aligned}
\end{equation}
where $\Theta = ( \boldsymbol{\pi}, (\boldsymbol{\mu}, \sigma) ).$
For categorical data, we use the categorical (multinoulli) mixture model; details are given in \cref{eq:multinoulli_mixture,eq:multinoulli_mixture_lcomp} in the Appendix.

For fair clustering baseline methods, we consider three existing methods: SFC, VFC, and FCA.
SFC \citep{backurs2019scalable} is a scalable fair clustering algorithm based on fairlet decomposition, yielding near-perfectly fair clustering.
VFC \citep{ziko2021variational} performs fair clustering by adding a variational fairness constraint based on the Kullback–Leibler (KL) divergence.
FCA \citep{kim2025fairclusteringalignment} is a recently developed state-of-the-art method that performs fair clustering by matching distributions across different groups.

In addition, for fairness-unaware clustering methods, we consider $K$-means and Gaussian mixture models learned with EM or GD (denoted MM-EM and MM-GD).

\paragraph{Evaluation}
For fairness metrics, we consider $\Delta$ and Balance, and for clustering utility metrics, we consider Cost (the sum of distances from each data point to its assigned cluster center). For FMC, we use the hard assignment map when computing the fairness metrics to ensure comparability, since SFC and VFC only yield hard assignments.

\paragraph{Initial parameters}
We set the initial value of mixture weights to \(\boldsymbol{\pi} = (1/K,\dots,1/K)^\top\).
The initial means $\boldsymbol{\mu}$ are set to the cluster centers obtained by the $K$-means algorithm with random initialization.
For the variance, we set $\sigma = 1.0.$
See the Appendix for additional implementation details.

\subsection{Performance Comparison}

\paragraph{Comparing FMC-EM and FMC-GD}

\cref{fig:10_seeds_del_bal} presents Pareto front plots of $\Delta$ versus Cost
(with standard deviation bands from five random initializations) for FMC-EM and FMC-GD on the Adult dataset.
It shows that FMC-EM achieves a better cost-fairness trade-off with smaller variance.
Accordingly, we use FMC-EM for subsequent experiments and, where appropriate, refer to it simply as FMC .

\begin{figure}[h]
    \centering
    \includegraphics[width=0.49\linewidth]{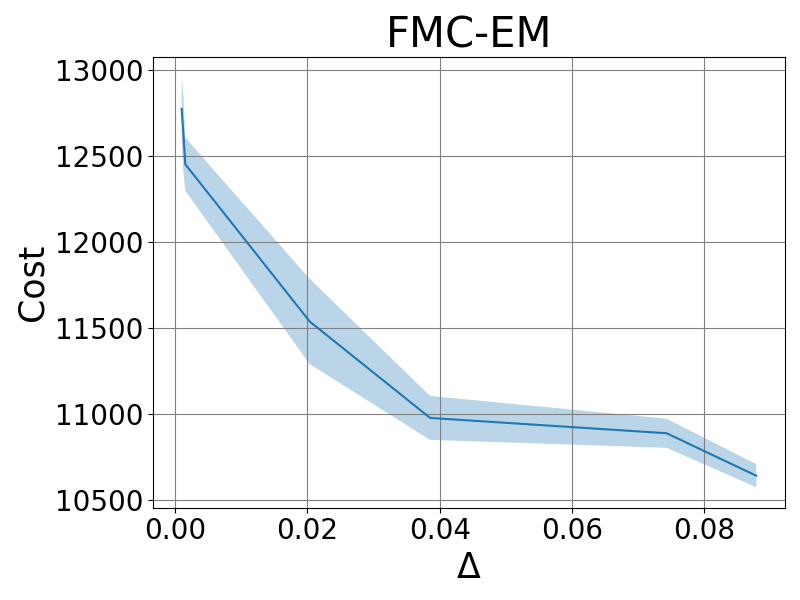}
    \includegraphics[width=0.49\linewidth]{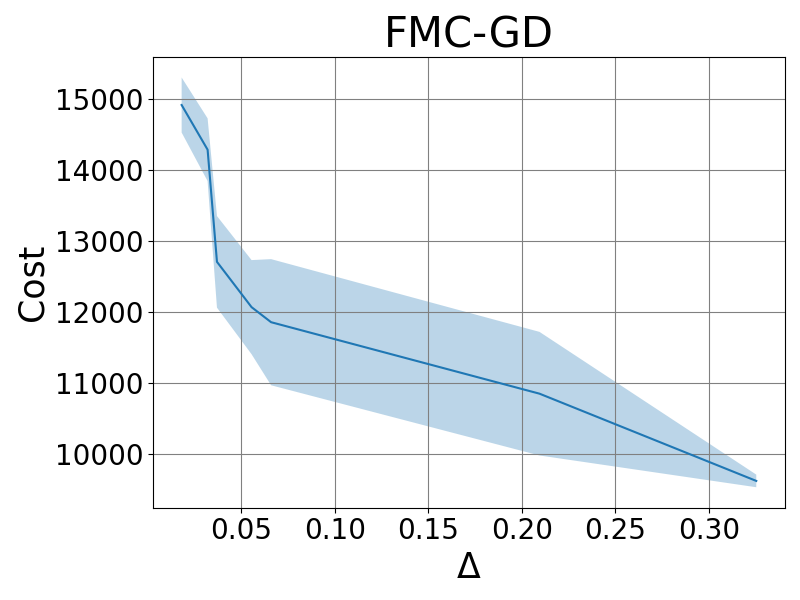}
    \normalsize
    \caption{
    Pareto front lines between $\Delta$ and Cost (with standard deviation bands obtained
    by five random initializations) of FMC-EM (left) and FMC-GD (right) on Adult dataset.
    See \cref{fig:10_seeds_remaining} in Appendix for similar results with respect to Balance and on other datasets.
    }
    \label{fig:10_seeds_del_bal}
\end{figure}

\paragraph{FMC vs. Baseline algorithms}

\cref{fig:three_L2_pareto_delta} shows the Pareto front plots of $\Delta$ versus Cost.
Note that SFC operates only at (near-perfect) fairness and therefore appears as a single point.
These results lead to two observations.

(i) FMC is better than SFC in terms of Cost at a near-perfect fairness level. 
This is not surprising, since FMC learns the cluster centers and the fair assignment map simultaneously, whereas SFC learns them sequentially. An additional advantage of FMC over SFC is the ability to control the fairness level, whereas SFC cannot.

(ii) VFC is slightly superior to FMC with respect to the cost–fairness trade-off.
This is expected because VFC minimizes Cost explicitly whereas FMC maximizes the log-likelihood, which is only indirectly related to Cost.
On the other hand, VFC fails to control the fairness level across the entire range.
For example, VFC does not provide a fair clustering with $\Delta<0.06$ on the Bank dataset.
Another advantage of FMC over VFC is robustness to data pre-processing methods.
As previously discussed, we apply $L_2$ normalization, since VFC fails to converge without it on the Credit and Census datasets (see \cref{fig:three_noL2_pareto,fig:census_noL2_pareto} in the Appendix).
We note that $L_2$ normalization is not a standard pre-processing step, since it maps points onto the unit sphere. In this sense, FMC is more broadly applicable.
See \cref{table:compare_cate_add_FCA} in the Appendix for additional comparisons between FMC and FCA.

\subsection{Analysis on Large-Scale Dataset}

This section provides experimental results on the Census dataset,
a large-scale dataset consisting of more than 2 million samples.

\paragraph{FMC vs. Baseline algorithms}

For FMC, we apply mini-batch learning with a sub-sampled $\Delta$, where the mini-batch and sub-sample sizes are set to 10\% of the full dataset.
\cref{fig:census_L2_pareto_delta} displays the results, showing that FMC is competitive with VFC and better than SFC. In addition, consistent with the results on moderate-scale data, FMC retains the advantage of controlling the fairness level across the entire range.
Moreover, FMC with mini-batch learning reduces computation time significantly (e.g., about 20x faster than VFC), as shown in \cref{table:large_data}.

\begin{figure}[h]
    \centering
    \includegraphics[width=0.49\linewidth]{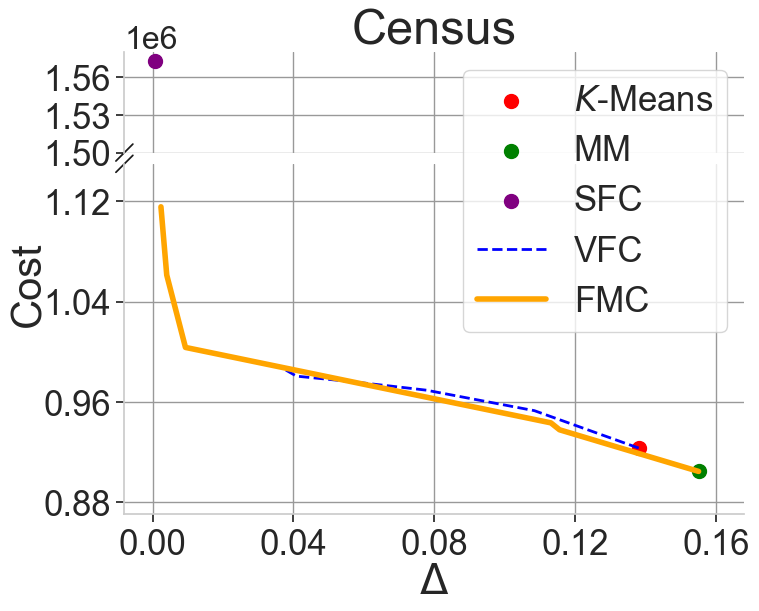}
    \includegraphics[width=0.49\linewidth]{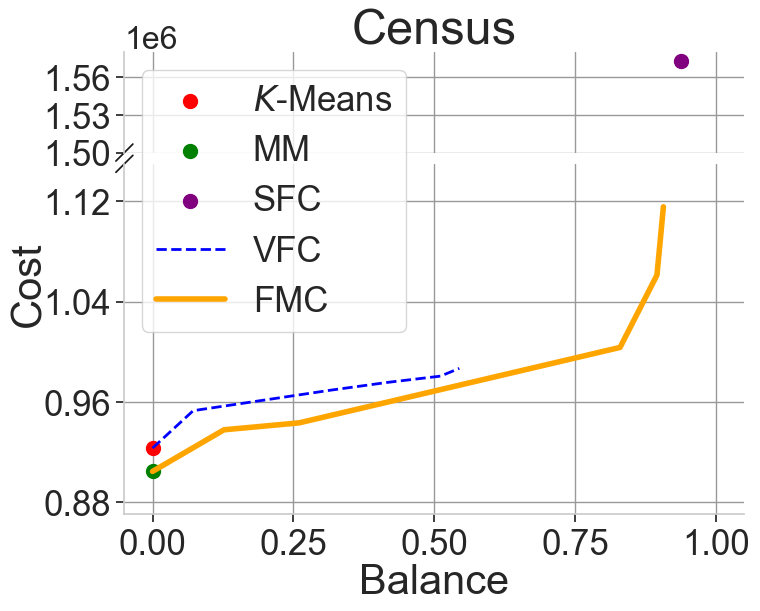}
    \normalsize
    \caption{Pareto front lines between $\{\Delta, \textup{ Balance}\}$ and Cost on Census dataset.
    See \cref{fig:census_noL2_pareto} in Appendix for the similar results without $L_2$ normalization.}
    \label{fig:census_L2_pareto_delta}
\end{figure}

\begin{table}[h]
    \centering
    \footnotesize
    \begin{tabular}{c|c}
        \toprule
        Algorithm & Time (std)
        \\
        \midrule
        VFC & 5053.0 (915.5)
        \\
        SFC & 2218.2 (535.1)
        \\
        FMC (mini-batch, $n = 0.1 N$) & 253.0 (12.1)
        \\
        \bottomrule
    \end{tabular}
    \caption{Comparison of the average computation time (seconds) with standard deviations (std) over five random 
    initials on Census dataset ($N=$ 2,458,285).}
    \label{table:large_data}
\end{table}

\paragraph{Mini-batch learning vs. Sub-sample learning}

We further investigate whether sub-sample learning empirically performs similarly to mini-batch learning by varying the sub-sample sizes in $\{1\%, 3\%, 5\%, 10\%\}.$
For example, when training with 5\% sub-samples, we evaluate the learned model on the full dataset (i.e., 
perform inference on the remaining 95\% of the data and then aggregate the inferred assignments with those of the training data).
\cref{table:large_data_subsampling_L2} compares the performance of
(i) sub-sample learning (sizes $\{1\%, 3\%, 5\%, 10\%\}$) and
(ii) mini-batch learning with batch size $10\%$.
Sub-sample learning with size $\ge 5\%$ yields performance similar to mini-batch learning but with lower computational cost. 
These results suggest that sub-sample learning is useful when the dataset is very large.

\begin{table}[h]
  \centering
  \footnotesize
  \begin{tabular}{l|cccc}
    \toprule
    Method $\left( \frac{n}{N} \right)$ & Cost ($\times 10^{5}$) & $\Delta$ & Balance & Time
    \\
    \midrule
    SS (1\%) & 11.167 & 0.005 & 0.851 & 53.7\%
    \\
    SS (3\%) & 10.988 & 0.004 & 0.894 & 65.8\%
    \\
    SS (5\%) & 10.973 & 0.003 & 0.883 & 71.4\%
    \\
    SS (10\%) & 10.916 & 0.003 & 0.881 & 77.2\%
    \\
    \midrule
    MB (10\%) & 10.857 & 0.002 & 0.896 & 100.0\%
    \\
    \bottomrule
  \end{tabular}
  
  \caption{Performances of mini-batch learning (MB) and sub-sample learning (SS) on Census dataset.
  The computation times of SS are measured by the relative ratio compared to the mini-batch learning with $10\%,$ and the Lagrangian $\lambda$ for each case is tuned to achieve the lowest Gap value i.e., $\Delta \approx 0.$
  See \cref{table:large_data_subsampling} in Appendix for the similar results without $L_2$ normalization.}
  \label{table:large_data_subsampling_L2}
\end{table}
\subsection{Categorical Data Analysis}
An additional advantage of FMC is the ability to incorporate categorical data into clustering by using the categorical (multinoulli) mixture model, whose formulation is given in \cref{eq:multinoulli_mixture,eq:multinoulli_mixture_lcomp} in the Appendix.
\cref{table:compare_cate_m=2,table:compare_cate_m=3} in the Appendix show that FMC performs well with categorical data.
Note that the baseline methods do not natively handle categorical variables; nontrivial modifications are required.

\subsection{Modification of FMC for Multinary Sensitive Attributes}
To apply FMC to multinary sensitive attributes, i.e., $M > 2,$ we propose the following modification of $\Delta$:
\begin{equation}\label{mutli_fairness_level}
    \footnotesize
    \max_{k \in [K]} 
    \frac{2}{M(M-1)} 
    \sum_{\substack{s_1,s_2 \in [M] \\ s_1 \neq s_2}} 
    \Delta_{s_1,s_2}
    ,
\end{equation}

where 
\begin{equation}
    \footnotesize
    \Delta_{s_1,s_2}:=\left\vert 
    \frac{
    \sum_{x_{i} \in \mathcal{D}^{(s_{1})}} \psi_k \left( x_i; \Theta\right)
    }{N_{s_1}} 
    -
    \frac{
    \sum_{x_{i} \in \mathcal{D}^{(s_{2})}} \psi_k \left( x_i; \Theta\right)
    }{N_{s_2}} 
    \right\vert.
\end{equation}
\normalsize
\cref{table:combined_m3} in the Appendix compares VFC and FMC on Bank and Credit datasets where the sensitive attribute has three categories. See the experimental details in the Appendix.
The results confirm that the modified FMC performs well.

\subsection{Parameter Impact Studies}
\paragraph{The number of clusters $K$}

We analyze how the number of clusters $K$ affects the performance of FMC, and we find that FMC performs well regardless of the choice of $K.$ 

That is, the negative log-likelihood decreases as $K$ increases, while the fairness level $\Delta$ remains sufficiently low.

See \cref{fig:ablation_num_k} in the Appendix for the results.

\paragraph{Choice of covariance structure in GMM}
We consider a diagonal covariance matrix instead of the isotropic one in the GMM to assess whether a more flexible covariance structure improves clustering.
\cref{fig:All_cov_compare} in the Appendix suggests that the diagonal covariance matrix is generally beneficial, which is an additional advantage of FMC over SFC and VFC.
\section{Discussion}

In this work, we have proposed FMC, a finite mixture model based fair clustering algorithm, which can be easily scaled-up through mini-batch and sub-sample learning. Moreover, FMC can handle both continuous and categorical data.

The main idea of FMC is to parameterize the assignment map. 
We leave the application of this idea to existing fair clustering algorithms (e.g., VFC) for future work.

Choosing the number of clusters $K$ is also important.
There is extensive research on estimating $K$ in finite mixture models \citep{schwarz1978estimating,biernacki2002assessing,richardson1997bayesian,miller2018mixture}.
We will investigate how to adapt these estimators to fair clustering in future work.


\section*{Acknowledgements}

This work was partly supported by the National Research Foundation of Korea (NRF) grant funded by the Korea government (MSIT) (No. 2022R1A5A7083908); 
Institute of Information\& communications Technology Planning \& Evaluation (IITP) grant funded by the Korea government (MSIT) (No.RS-2022-II220184, Development and Study of AI Technologies to Inexpensively Conform to Evolving Policy on Ethics);
Institute of Information \& communications Technology Planning \& Evaluation (IITP) grant funded by the Korea government (MSIT) [NO.RS-2021-II211343; Artificial Intelligence Graduate School Program (Seoul National University)],
and the National Research Foundation of Korea (NRF) grant funded by the Korea government (MSIT) (RS-2025-00556079).

\bibliography{ref}

\clearpage
\appendix
\onecolumn

\begin{center}
    \textbf{\huge Appendix for `Fair Model-based Clustering'}
\end{center}

\section{Proof of \cref{prop:generalization}}

In this proof, we overload notations for clarity by reusing several symbols (e.g., $x_{i}, f, \theta, \Theta, \mathcal{N}$) with meanings that differ from those in the main paper.
However, we believe that these overloaded symbols are only relevant within the context of this proof, so no confusion would arise.

\subsection{Auxiliary Definitions \& Lemmas}

\begin{definition}[Rademacher complexity of a function class \cite{mohri2018foundations}]
    Let \(\mathcal{F}\) be a function class of real-valued functions over some domain \(\mathcal{X}\). 
    Let \(\mathcal{D}_m=\left\{ x_1,\dots,x_m \right\}\) be \(m\)-samples drawn i.i.d. from a fixed distribution \(\mathcal{D}\). 
    Let \(\sigma_1,\dots,\sigma_m\) be independent random variables following the Rademacher distribution, 
    i.e. \(\Pr(\sigma_i=1)=\Pr(\sigma_i=-1)=1/2\) for \(\forall i\). 
    Then the empirical Rademacher complexity of \(\mathcal{F}\) is defined as:
    \begin{equation}
        \hat{\mathcal{R}}_m(\mathcal{F})=\frac{1}{m}\mathbb{E}_\sigma\left[ \underset{f\in\mathcal{F}}{\sup}\sum_{i=1}^{m}\sigma_if(x_i) \right]
    \end{equation}
    The Rademacher complexity of \(\mathcal{F}\) is defined as: 
    \begin{equation}
        \mathcal{R}_m(\mathcal{F})=\mathbb{E}_D\left[ \hat{\mathcal{R}}_m(\mathcal{F}) \right]
    \end{equation}
\end{definition}

\begin{lemma}[Rademacher generalization bound \cite{shalev2014understanding}]
\label{lem:rademacher}
    Let \(\mathcal{F}:=l\circ\mathcal{H}:=\left\{ z\mapsto l(h,z):h\in\mathcal{H} \right\}\)
    For given \(f\in\mathcal{F}\), let \(L_\mathcal{D}(f)=\mathbb{E}_{z\sim\mathcal{D}}[f(z)]\) and \(L_S(f)=\frac{1}{m}\sum_{i=1}^{m}f(z_i)\).
    Also, assume that for all \(z\) and \(h\in\mathcal{H}\) we have that \(|l(h,z)|\le c\). 
    Then, with probability of at least \(1-\delta\), for all \(h\in\mathcal{H}\), 
    \begin{equation}
        |L_\mathcal{D}(h)-L_S(h)|\le 2\hat{\mathcal{R}}_m(l\circ\mathcal{H})+4c\sqrt{\frac{2\log(4/\delta)}{m}}.
    \end{equation}
\end{lemma}

Let $\mathcal{N}(u, \mathcal{F}, d)$ denotes the $u$-covering number of $\mathcal{F}$ with respect to metric $d$. 

\begin{lemma}[Dudley's theorem \cite{wolf2023mathematical}]
\label{lem:dudley}
    Assume that for any \(g\in\mathcal{G}\), \(||g||_\infty\le\gamma_0\) holds for some \(\gamma_0>0\). 
    Then, 
    \begin{equation*}
        \hat{R}_m(\mathcal{G})\le\underset{\epsilon\in[0,\gamma_0/2)}{\inf}4\epsilon+\frac{12}{\sqrt{m}}\int_{\epsilon}^{\gamma_0}\left( \log\mathcal{N}(\beta,\mathcal{G},||\cdot||_{m,2} \right)^{1/2}d\beta.
    \end{equation*}
\end{lemma}

\begin{lemma}[Upper bound of Rademacher complexity]
\label{lem:Rad_ub}
    Let $\mathcal{F}$ be a class parametric functions $\mathcal{F}=\{ f(\cdot;\theta): \theta\in\Theta \}\in[0,1]$, where $||\theta||_2\le B$ holds for $B>0$. 
    Assume that $f$ is $L$-Lipschitz, i.e. $|f(\cdot;\theta)-f(\cdot;\theta')|\le L||\theta-\theta'||_2$. 
    Then, we can bound the Rademacher complexity for some constant $C>0$ as follows: 
    \begin{equation}
        \hat{\mathcal{R}}_n(\mathcal{F})\le C LB\sqrt{\frac{d}{n}}. 
    \end{equation}
\end{lemma}

\begin{proof}
    From the Lipschitzness, we have the following inequality: 
    \begin{equation}
        \mathcal{N}(u, \mathcal{F}, ||\cdot||_{n,2})\le \mathcal{N}(\frac{u}{L},\Theta, ||\cdot||_2), 
    \end{equation}
    where $\mathcal{N}(u, \mathcal{F}, d)$ denotes the $u$-covering number of $\mathcal{F}$ with respect to metric $d$. 

    Also, we know that the covering number of a $d$-dimensional Euclidean ball can be bounded \cite{vershynin2018high} as follows: 
    \begin{equation}
        \log\mathcal{N}(u,\Theta,||\cdot||_2)\le d\log(\frac{3B}{u})
    \end{equation}
    for $0<u\le B$. 

    Now we can apply \cref{lem:dudley} as follows: 
    \begin{align}
        \hat{\mathcal{R}}_n(\mathcal{F})&\le\frac{12}{\sqrt{n}}\int_0^{LB}\sqrt{\log\mathcal{N}(u,\mathcal{F},||\cdot||_{n,2})}\;du\\
        &\le\frac{12}{\sqrt{n}}\int_0^{LB}\sqrt{d\log\left(\frac{3LB}{u}\right)}\;du\\
        &=12LB\sqrt{\frac{d}{n}}\int_0^1\sqrt{\log(\frac{3}{t})}\;dt\\
        &=:CLB\sqrt{\frac{d}{n}}.
    \end{align}
\end{proof}

\subsection{Main Proof}

\begin{proof}[Proof of \cref{prop:generalization}]
    We start with the following inequality: 
    \begin{equation}
        a_k\le b_k+|a_k-b_k|\le\underset{l}{\max}b_l+\underset{l}{\max}|a_l-b_l|,
    \end{equation}
    holds for all $k$. 
    By taking maximum, we have: 
    \begin{equation}
        \underset{k}{\max}a_k\le\underset{k}{\max}b_k+\underset{k}{\max}|a_k-b_k|.
    \end{equation}
    Hence, we have:
    \begin{equation}
        \left\lvert \underset{k}{\max}a_k-\underset{k}{\max}b_k \right\rvert\le\underset{k}{\max}|a_k-b_k|. 
    \end{equation}
    Using the above inequality and the fact that $||a'-b'|-|a-b||\le|a-a'|+|b-b'|$, we have the following inequality: 
    \begin{align}
        &|\Delta(\Theta;\mathcal{D}_n)-\Delta(\Theta)|\\
        &\le\underset{k}{\max}\left\lvert \left\lvert \frac{1}{n_1}\sum_{i=1}^{n_1}\psi_k(x_i^{(1)};\Theta)-\frac{1}{n_2}\sum_{i=1}^{n_2}\psi_k(x_i^{(2)};\Theta) \right\rvert-\left\lvert \mathbb{E}[\psi_k(x_i^{(1)};\Theta)]-\mathbb{E}[\psi_k(x_i^{(2)};\Theta)] \right\rvert \right\rvert\\
        &\le\underset{k}{\max}\left\lvert \frac{1}{n_1}\sum_{i=1}^{n_1}\psi_k(x_i^{(1)};\Theta)-\mathbb{E}[\psi_k(x_i^{(1)};\Theta)] \right\rvert+\underset{k}{\max}\left\lvert \frac{1}{n_2}\sum_{i=1}^{n_2}\psi_k(x_i^{(2)};\Theta)-\mathbb{E}[\psi_k(x_i^{(2)};\Theta)] \right\rvert.
    \end{align}
    Here, the expectations are taken over the training data $\mathcal{D}$.

    Note that the responsibility can be represented using the softmax function as:
    \begin{equation}
        \psi_k(x;\Theta)=\operatorname{softmax}(s_k(x;\Theta)),
    \end{equation}
    where $s_k(x;\Theta)=\pi_k\mathcal{N}(x;\mu_k,\Sigma_k)$ and $\mathcal{N}$ is the Gaussian density function with parameters $\mu_{k}, \Sigma_{k}, k \in [K].$
    Since $||J_{\operatorname{softmax}}||_{\textup{op}}\le\frac{1}{2}$ (where $||J_{\operatorname{softmax}}||_\textup{op}$ is the operator norm of the Jacobian of the softmax), 
    we have 
    \begin{equation}
        ||\nabla_\theta\psi_k(x;\Theta)||\le \frac{1}{2}||\nabla_\theta s_k(x;\Theta)||_F. 
    \end{equation}
    Moreover, we have for all $\Theta \in \Phi_{\xi, \zeta, \nu}$ that,
    \begin{align}
        \left\lVert \frac{\partial s_j}{\partial\pi_j} \right\rVert\le\frac{1}{\xi} := A_{\pi},
        \left\lVert \frac{\partial s_j}{\partial \mu_j} \right\rVert\le \zeta^{-1}(x_{\textup{max}} + \nu) =: A_\mu,
        \left\lVert \frac{\partial s_j}{\partial \Sigma_j} \right\rVert_F\le\frac{1}{2} \zeta^{-1}+\frac{1}{2} \zeta^{-2}(x_{\textup{max}} + \nu)^2=: A_\Sigma.
    \end{align}
    Hence, we have:
    \begin{equation}
        ||\nabla_\theta s_k(x;\Theta)||_F\le\sqrt{K(A_\pi^2+A_\mu^2+A_\Sigma^2)}=:C_s.
    \end{equation}
    To sum up, $\psi_k(x;\Theta)$ is Lipschitz with respect to 
    \begin{equation}
        L:=\underset{x, \Theta \in \Phi_{\xi,\zeta,\nu}}{\sup}||\nabla_\theta \psi_k(x;\Theta)||\le\frac{1}{2}C_s,
    \end{equation}
    where $C_{s}$ is a constant only depending on $\xi, \zeta, \nu.$
    Now from \cref{lem:rademacher,lem:Rad_ub}, the following inequality holds for all $k$ with probability at least $1-\delta$: 
    \begin{equation}
        \underset{\Theta \in \Phi_{\xi,\zeta,\nu}}{\sup}\left\lvert \frac{1}{n_1}\sum_{i=1}^{n_1}\psi_k(x_i^{(1)};\Theta)-\mathbb{E}[\psi_k(x_i^{(1)};\Theta)] \right\rvert
        \le 2C'L\sqrt{\frac{d}{n_1}}+4\sqrt{\frac{2\log(4/\delta)}{n_1}},
    \end{equation}
    where $C'$ is a constant depending on $\xi, \zeta, \nu.$
    Therefore, we can conclude that, with probability at least $1-\delta$:
    \begin{equation}
       \sup_{\Theta \in \Phi_{\xi,\zeta,\nu}} |\Delta(\Theta;\mathcal{D}_n)-\Delta(\Theta)|
        \le 4C\sqrt{\frac{d}{n'}}+8\sqrt{\frac{2\log(8/\delta)}{n'}}, 
    \end{equation}
    where $n'=\min\{n_1,n_2\}$ and $C=C'L.$
\end{proof}

\section{Small Disparity Between Soft and Hard Assignments}

As previously discussed, one can obtain the hard‐assignment map $\psi_{hard}(x_i;\Theta)$ for a given $x_{i}$ from the soft assignment map:
$\psi_{hard}(x_i;\Theta)=\argmax_{k \in [K]} \{ \psi_{k}(x_i; \Theta) \}.$
\cref{rmk:hard_soft} below shows that, under a mild separation condition in high dimension, 
the soft assignment map is close to the corresponding hard assignment map.

\begin{remark}\label{rmk:hard_soft}
    Recall that for an isotropic Gaussian mixture model,
    $\psi_k(x;\Theta) \propto \exp \bigl(-\|x-\mu_k\|^2/(2\sigma^2)\bigr).$
    Let 
    $k_1=\arg\min_k\|x-\mu_k\|$
    be the cluster index of the closest cluster center,
    and let
    $k_2=\arg\min_{k\neq k_1}\|x-\mu_k\|$
    be the cluster index of the second-closest cluster center.
    Let $D = \|x-\mu_{k_2}\|^2-\|x-\mu_{k_1}\|^2>0.$
    Then for each $x,$ we have
    \begin{equation}
        \psi_{k_1}(x;\Theta)
        = \frac{\exp(-\|x-\mu_{k_1}\|^2/(2\sigma^2))}
        {\sum_{l=1}^K\exp(-\|x-\mu_{l}\|^2/(2\sigma^2))}
        \ge \frac{1}{1+(K-1)\exp(-D/(2\sigma^2))}\,. 
    \end{equation}
    Hence, as $d\to\infty$ or $D \gg 2\sigma^2\log(K-1)$, we have
    $\exp(-D/(2\sigma^2))\to0\,$ which implies 
    $\psi_{k_1}(x;\Theta)\to1\,. $

    In other words, in high‐dimensional or well‐separated regimes, the soft assignment map $\psi_{k}(x_i; \Theta)$ becomes very close to the hard assignment map $\psi_{hard}(x_i;\Theta),$ and thus not much loss occurs when using the hard assignment map instead of the soft assignment map.
\end{remark}
\section{Detailed Explanations on EM Algorithm}
\citet{wu1983convergence} showed that the EM algorithm converges to a stationary point of the log-likelihood function \(\ell(\Theta \mid \mathcal{D})\) under mild regularity conditions.
That is, iterating the EM algorithm yields parameter estimates that locally maximize \(\ell(\Theta \mid \mathcal{D})\).
Let \(Y = (\mathcal{D}, Z)\) be a random vector that follows a joint distribution of \(\mathcal{D}\) and \(Z\).
Then, instead of maximizing the log-likelihood directly, the EM algorithm tries to find parameters that maximizes the complete-data log-likelihood \(\ell_{comp}(\Theta \mid Y) = \log f(Y \mid \Theta)\).
To be more specific, at \(t\)th iteration, \(Q\) function is calculated in E-step, and parameters are updated to maximize the \(Q\) function in M-step, where the $Q$ function is defined as
\begin{align}\label{Q_func}
    Q\bigl(\Theta \mid \Theta^{(t)}\bigr) = \mathbb{E}_{Z \mid \mathcal{D}; \Theta^{{[t]}}} \ell_{comp}(\Theta \mid Y),
\end{align}
where the complete-data log-likelihood $\ell_{comp}(\Theta|Y)$ is defined as:
\begin{equation}
    \ell_{comp}(\Theta \mid Y) = \sum_{i=1}^N \sum_{k=1}^K \mathbb{I}(Z_{i} = k) \left\{ \log \pi_{k} + \log f(x_i; \theta_k) \right\}.
\end{equation}
In E-step, we calculate the responsibility as
$$
\psi_k(x_i;\Theta)
= P\bigl(Z_i = k \mid x_i; \Theta^{[t]}\bigr)
= \frac{\pi_k^{[t]}\,f(x_i;\theta_k^{[t]})}
       {\sum_{l=1}^K \pi_l^{[t]}\,f(x_i;\theta_l^{[t]})}, k \in [K],
$$
so that the $Q$ function is calculated as
\begin{equation}\label{eq:emm_q_function}
    Q\bigl(\Theta \mid \Theta^{[t]}\bigr) = \sum_{i=1}^N \sum_{k=1}^K \psi_k(x_i;\Theta^{[t]}) \bigl[\log\pi_k + \log f(x_i;\theta_k)\bigr].
\end{equation}
Then in M-step, we find $\Theta^{[t+1]}$ that maximizes the $Q$ function:
$$
\Theta^{[t+1]} = \argmax_{\Theta} Q\bigl(\Theta \mid \Theta^{[t]}\bigr)
= \argmax_{\Theta} \mathbb{E}_{Z \mid \mathcal{D};\,\Theta^{[t]}}\bigl[\ell_{{comp}}(\Theta \mid Y)\bigr]
= \argmax_{\Theta} \sum_{i=1}^N \sum_{k=1}^K
  \psi_k(x_i;\Theta^{[t]}) \bigl[\log\pi_k + \log f(x_i;\theta_k)\bigr].
$$

\clearpage
\section{FMC-EM for Mini-Batch Learning}

We here describe the algorithm for FMC-EM for mini-batch learning with sub-sampled $\Delta$ (i.e., applying the mini-batch GEM to FMC).
\cref{alg:fairEM_minibatch} below displays the overall algorithm of FMC-EM.
Define $ \ell_{comp, i}(\Theta \mid Y) = \sum_{k=1}^{K} \mathbb{I}(z_{i} = k) \{ \log \pi_{k} + \log f(x_{i}; \theta_{k}) \} $
and
$
Q_{fair, i}(\Theta \mid \Theta^{[0]} ; \lambda) 
= \sum_{k=1}^K \psi_k(x_i;\Theta^{[0]}) \bigl[\log\pi_k + \log f(x_i;\theta_k)\bigr]
- \frac{\lambda}{N} \Delta(\Theta; \mathcal{D}_{n}) $ for $i \in [N].
$

\begin{algorithm}[h]
    \caption{FMC-EM (Expectation-Maximization) for mini-batch learning with sub-sampled $\Delta$ \\
    In practice, we set \((T, R, \gamma) = (200, 10, 10^{-2})\)}\label{alg:fairEM_minibatch}
    \begin{algorithmic}
        \STATE \textbf{Input:} Dataset \(\mathcal{D} = \{x_1,\dots,x_N\}\), Sub-sampling size $n \le N,$ Number of clusters \(K\), Lagrange multiplier \(\lambda\), Maximum numbers of iterations \(T, R\) and Learning rate \(\gamma\).
        \STATE \textbf{Initialize:} \(\Theta^{[0]} = (\boldsymbol{\eta}^{[0]}, \boldsymbol{\theta}^{[0]})\)
        \STATE \textbf{Initialize:} The sub-sampled data $\mathcal{D}_{n} \subseteq \mathcal{D}$
        \STATE \textbf{Calculate:} $ A_{i}^{0}(\Theta) = Q_{fair, i}(\Theta \mid \Theta^{[0]} ; \lambda)$ for all $i \in [N]$
        \WHILE{$\sum_{i=1}^{N} A_{i}(\Theta)$ has not converged and $t < T$}
        \STATE Randomly select a mini-batch (index set) $I_{t} \subset [N]$
        \STATE Compute $Q_{fair, i}(\Theta \mid \Theta^{[t]} ; \lambda)$ for $i \in I_{t}$ and $\Delta(\Theta; \mathcal{D}_{n})$
        \STATE Set
        \begin{equation*}
            A_{i}^{t}(\Theta) =
            \begin{cases}
                Q_{fair, i}(\Theta \mid \Theta^{[t]} ; \lambda) & \textup{ if } i \in I_{t}
                \\ 
                A_{i}^{t-1}(\Theta) & \textup{ otherwise }
            \end{cases}
        \end{equation*}
            \WHILE{$r < R$}
                \STATE $\boldsymbol{\theta}_{(r+1)}^{[t]} \gets \boldsymbol{\theta}_{(r)}^{[t]} + \gamma\frac{\partial A_{i}^{t}(\Theta)}{\partial \boldsymbol{\theta}}$
                \STATE $\boldsymbol{\eta}_{(r+1)}^{[t]} \gets \boldsymbol{\eta}_{(r)}^{[t]} + \gamma\frac{\partial A_{i}^{t}(\Theta)}{\partial \boldsymbol{\eta}}$
                \STATE \(\boldsymbol{\pi}_{(r+1)}^{[t]} \gets \text{softmax } (\boldsymbol{\eta}_{(r+1)}^{[t]}) \)
                \STATE $r \gets r+1$
            \ENDWHILE
            \STATE \(\boldsymbol{\theta}^{[t+1]} \gets \boldsymbol{\theta}^{[t]}_{(R)}\)
            \STATE \(\boldsymbol{\pi}^{[t+1]} \gets \boldsymbol{\pi}^{[t]}_{(R)}\)
            \STATE $t \gets t+1$
        \ENDWHILE
    \end{algorithmic}
\end{algorithm}

\section{Theoretical Property of FMC-EM for Mini-Batch Learning With Gaussian Mixture Model}

Following Theorem 1 of \citet{karimi:hal-02334485}, which shows the theoretical convergence of mini-batch EM,
we can derive a similar result for our FMC-EM for mini-batch learning.
The following corollary shows that, the `mini-batch learning with sub-sampled $\Delta$' for FMC-EM
with the Gaussian mixture model is theoretically valid in the sense that the log-likelihood of learned parameters using mini-batches converges almost surely to a stationary point.
In the following corollary, we abbreviate $\Phi_{\xi,\zeta,\nu}$ by $\Phi$ for notational simplicity. 

\begin{corollary}\label{cor:converge}
    Consider the Gaussian mixture for $\mathbf{f}(\cdot; \Theta)$ with $\Theta\in\Phi$. 
    Let $(\Theta^{[t]})_{t \ge 1}$ be a sequence initialized by $\Theta^{[0]}$ following the iterative process in \cref{alg:fairEM_minibatch}. 
    For any $\Theta,$ let $ \bar A^t(\Theta) = \sum_{i=1}^N \Lambda_i^t(\Theta):= -\sum_{i=1}^N A_i^t(\Theta).$ 
    Denote the gradient step of the M-step as follows: 
    \begin{equation}\label{eq:grad_step}
        \Theta^{[t]}=\Pi_{\Phi}(\Theta^{[t-1]}-\gamma\nabla\bar{A}^t(\Theta^{[t-1]})), 
    \end{equation}
    where $\Pi_{\Phi}$ is a projection operator to the parameter space $\Phi$ and $\gamma>0$ is the learning rate. 
    For $0<\gamma<\frac{2\zeta}{N}$, the followings hold: 
    \begin{enumerate}
        \item[(i)] $( \ell(\Theta^{[t]} \mid \mathcal{D}) )_{t \ge 1}$ converges almost surely.
        \item[(ii)] $(\Theta^{[t]})_{t \ge 1}$ satisfies the Asymptotic Stationary Point Condition, i.e.,
        \begin{equation}
            \underset{t\to\infty}{\limsup}\underset{\Theta\in\Phi}{\sup}\frac{\langle \nabla\ell(\Theta^{[t]}\mid \mathcal{D}), \Theta-\Theta^{[t]} \rangle}{||\Theta-\Theta^{[t]}||}\le 0, 
        \end{equation}
        That is, $\Theta^{[t]}$ asymptotically converges to a stationary point. 
    \end{enumerate}
\end{corollary}

\begin{proof}
    We prove (i) and (ii) sequentially as follows.
    
    \textbf{[Proof of (i)]}    
    First, we note that all regularity conditions required for mini-batch EM hold for Gaussian densities (see \citet{karimi:hal-02334485} for the detailed regularity conditions; we omit the conditions in this proof for brevity).
    Furthermore, as the Gap $\Delta(\Theta; \mathcal{D}_{n})$ is computed on a fixed $\mathcal{D}_{n},$ we only need to focus on the $\ell_{comp, i}$ term in $Q_{fair, i}.$
    Let $\Omega_{i} = -Q_{fair, i}$ and $\iota(\Theta) = -\ell (\Theta \mid \mathcal{D}).$
    Then, we can apply the same arguments in the proof of Theorem 1 of \citet{karimi:hal-02334485}, as below. 
    From the smoothness regularity condition, we have that
    \begin{equation}
        \bar{A}^t(\Theta^{t})
        \le\bar{A}^t(\Theta^{[t-1]})+\langle \nabla\bar{A}^t(\Theta^{[t-1]}),\Theta^{[t]}-\Theta^{[t-1]} \rangle+\frac{L}{2}||\Theta^{[t]}-\Theta^{[t-1]}||^2,  
    \end{equation}
    where $L:=N/\zeta$. 
    Due to the definition of the projection operator, we have that
    \begin{equation}
        \langle \Theta^{[t]}-(\Theta^{[t-1]}-\gamma\nabla\bar{A}^t(\Theta^{[t-1]})), \Theta^{[t-1]}-\Theta^{[t]}\rangle\le 0, 
    \end{equation}
    which leads to 
    \begin{equation}\label{eq:proj_prop}
        \langle \nabla\bar{A}^t(\Theta^{[t-1]}),\Theta^{[t]}-\Theta^{[t-1]} \rangle\le -\frac{1}{\gamma}||\Theta^{[t]}-\Theta^{[t-1]}||^2.
    \end{equation}
    Hence, we have:
    \begin{equation}\label{eq:descent_prop_grad}
        \bar{A}^t(\Theta^{[t]})\le\bar{A}^t(\Theta^{[t-1]})-\nu||G^{t-1}_\gamma(\Theta^{[t-1]})||^2, 
    \end{equation}
    where $G^{t-1}_\gamma(\Theta)=\frac{1}{\gamma}(\Theta-\gamma\nabla\bar{A}^{[t-1]}(\Theta))$ and $\nu=\gamma-\frac{L}{2}\gamma^2>0$. 
    For any $t\ge1$ and $\Theta$, we have that
    \begin{equation}
       \bar A^t(\Theta)=\bar A^{t-1}(\Theta)
      +\sum_{i\in I_{t}}\Omega_{i}(\Theta \mid \Theta^{[t-1]})
      -\sum_{i\in I_{t}}\Lambda_i^{t-1}(\Theta). 
    \end{equation}
    From \cref{eq:descent_prop_grad}, we have that
    \begin{equation}
       \bar A^t(\Theta^{[t-1]})\le\bar A^{t-1}(\Theta^{[t-1]})
      +\sum_{i\in I_{t}}\Omega_{i}(\Theta^{[t-1]} \mid \Theta^{[t-1]})
      -\sum_{i\in I_{t}}\Lambda_i^{t-1}(\Theta^{[t-1]})-\nu||G_\gamma^{t-1}(\Theta^{[t-1]})||^2. 
    \end{equation}
    Since 
    $ \Omega_{i}(\Theta^{[t]} \mid \Theta^{[t-1]})=\iota_i(\Theta^{[t-1]}) $ for $i\in I_{t}$
    and
    $ \iota_i(\Theta^{[t-1]})-\Lambda_i^{t-1}(\Theta^{[t-1]}) \le 0$ for $i \in [N],$
    it follows that
    \begin{equation}
        \bar A^t(\Theta^{[t]})\le\bar A^{t-1}(\Theta^{[t-1]})
        +\sum_{i\in I_{t}}\iota_i(\Theta^{[t-1]})
        -\sum_{i\in I_{t}}\Lambda_i^{t-1}(\Theta^{[t-1]})-\nu||G_\gamma^{t-1}(\Theta^{[t-1]})||^2.
    \end{equation}
    Let $\mathcal{F}_{t-1}=\sigma(I_j:j\le k-1)$ be the filtration generated by the sampling of the indices. 
    Then, we have that
    \begin{equation}
        \mathbb{E}\left[\Lambda_{I_t}^{t-1}(\Theta^{[t-1]})-\iota_{I_t}(\Theta^{[t-1]})\mid\mathcal{F}_{t-1}\right]=\frac{p}{N}\left( \bar{A}^{t-1}(\Theta^{[t-1]})-\iota(\Theta^{[t-1]}) \right), 
    \end{equation}
    which leads to 
    \begin{equation}
        \mathbb{E}\left[\bar{A}^{t}(\Theta^{[t]})\vert\mathcal{F}_{t-1}\right]\le\bar{A}^{t-1}(\Theta^{[t-1]})-\frac{p}{N}\left( \bar{A}^{t-1}(\Theta^{[t-1]})-\iota(\Theta^{[t-1]}) \right)-\nu||G_\gamma^{t-1}(\Theta^{t-1})||^2.
    \end{equation}

    From the Robbins-Siegmund lemma \citep{robbins1971convergence}, we have that
    \begin{align}
        \underset{t\to\infty}{\lim}\bar{A}^{t}(\Theta^{[t]})-\iota(\Theta^{[t]})&=0\textup{ a.s.}\\
        \sum_{t}\nu||\nabla\bar{A}^{t}(\Theta^{[t-1]})||&<\infty\textup{ a.s.}
    \end{align}

    \textbf{[Proof of (ii)]}
    Similar to \citet{karimi:hal-02334485}, we define for all $t\ge 0$, 
    \begin{equation}
        \bar{h}^{t}:\Theta\to\sum_{i=1}^{N}\Lambda_i^{t}(\Theta)-\iota_i(\Theta).
    \end{equation}
    Note that $\bar{h}^t$ is $L$-smooth on $\Phi$. 
    Similarly, we have the following inequality using Lemma 1.2.3 in \citet{nesterov2013gradient}: 
    \begin{equation}
        ||\nabla\bar{h}^t(\Theta^{[t]})||_2^2\le 2L\bar{h}^{t}(\Theta^{[t]})\to 0.
    \end{equation}
    Then, we have that
    \begin{align}
        &\langle \nabla\iota(\Theta^{[t]}),\Theta-\Theta^{[t]} \rangle\\
        &=\langle \nabla\bar{A}^{t}(\Theta^{[t-1]})+\eta^{t},\Theta-\Theta^{[t]} \rangle
        -\langle \eta^{t},\Theta-\Theta^{[t]} \rangle
        +\langle \nabla\bar{A}^{t}(\Theta^{[t]})-\nabla\bar{A}^{t}(\Theta^{[t-1]}),\Theta-\Theta^{[t]} \rangle
        -\langle \nabla\bar{h}^{t}(\Theta^{[t]}),\Theta-\Theta^{[t]} \rangle, 
    \end{align}
    where $\eta^{t}=\frac{1}{\gamma}(\Theta^{[t]}-(\Theta^{[t-1]}-\gamma\nabla\bar{A}^t(\Theta^{[t-1]})))$. 
    From \cref{eq:grad_step}, we have that
    \begin{align}
        \langle \nabla\bar{A}^{t}(\Theta^{[t-1]})+\eta^t,\Theta-\Theta^{[t]} \rangle&\ge 0,\\
        \langle \eta^t,\Theta-\Theta^{[t]} \rangle&\le 0, 
    \end{align}
    and from the smoothness, we have that
    \begin{align}
        ||\langle \nabla\bar{A}^t(\Theta^{[t]})-\nabla\bar{A}^t(\Theta^{[t-1]}),\Theta-\Theta^{[t]} \rangle||
        &\le||\nabla\bar{A}^{t}(\Theta^{[t]})-\nabla\bar{A}^{t}(\Theta^{[t-1]})||\cdot||\Theta-\Theta^{[t]}||\\
        &\le L||\Theta^{[t]}-\Theta^{[t-1]}||\cdot||\Theta-\Theta^{[t]}||.
    \end{align}
    Therefore, we have that
    \begin{align}
        \langle \nabla\iota(\Theta^{[t]}),\Theta-\Theta^{[t]} \rangle
        &\ge -L||\Theta^{[t]}-\Theta^{[t-1]}||\cdot||\Theta-\Theta^{[t]}|| -\langle \nabla\bar{h}^{t}(\Theta^{[t]}),\Theta-\Theta^{[t]} \rangle\\
        &\ge -L||\Theta^{[t]}-\Theta^{[t-1]}||\cdot||\Theta-\Theta^{[t]}|| -||\nabla\bar{h}^t(\Theta^{[t]})||\cdot||\Theta-\Theta^{[t]}||.
    \end{align}
    Note that from the definition of the projection operator, we have that
    \begin{equation}
        ||\Theta^{[t]}-\Theta^{[t-1]}||\le\gamma||\nabla\bar{A}^{t}(\Theta^{t-1})||
    \end{equation}
    which leads to 
    \begin{equation}
        \sum_{t=1}^{\infty}||\Theta^{[t]}-\Theta^{[t-1]}||^2\le\gamma^2\sum_{t=1}^{\infty}||\nabla\bar{A}^{t}(\Theta^{[t-1]})||^2<\infty.
    \end{equation}
    Therefore, as same as \citet{karimi:hal-02334485}, we conclude that
    \begin{equation}
        \underset{t\to\infty}{\liminf}\underset{\Theta\in\Phi}{\inf}\frac{\langle \nabla\iota(\Theta^{[t]}),\Theta-\Theta^{[t]} \rangle}{||\Theta-\Theta^{[t]}||}\ge 0, 
    \end{equation}
    in other words, 
    \begin{equation}
        \underset{t\to\infty}{\limsup}\underset{\Theta\in\Phi}{\sup}\frac{\langle \nabla\ell(\Theta^{[t]}),\Theta-\Theta^{[t]} \rangle}{||\Theta-\Theta^{[t]}||}\le 0.
    \end{equation}
\end{proof}

\clearpage
\section{Derivation of the Gradients in \cref{alg:fairGD}}\label{Appd_A}

Here, we derive the gradients of $\ell(\Theta; \lambda)$ with respect to $\mu_{k}, k \in [K]$ and $\eta_{k}, k \in [K]$ in \cref{alg:fairGD}.
For simplicity, we only consider \(M=2\) and \(\Sigma = \mathbb{I}_d\) case.
Let
\begin{equation}
    q(\Theta; k) = \left\vert \frac{\sum_{x_{i} \in \mathcal{D}^{(1)}} \psi_{k}\left(x_i ; \Theta \right)}{N_{1}} - \frac{\sum_{x_j \in \mathcal{D}^{(2)}} \psi_{k} \left(x_j; \Theta \right)}{N_{2}} \right\vert
\end{equation}
and $\tilde{k} = \argmax_{k \in [K]} q(\Theta;k).$

\begin{itemize}
    \item Gradient with respect to \(\mu_k\)

    \begin{gather}
        \frac{\partial \ell(\Theta; \lambda)}{\partial \mu_k} = \frac{\partial \ell(\Theta \mid \mathcal{D})}{\partial \mu_k} - \lambda \frac{\partial 
        \Delta(\Theta)}{\partial \mu_k} \\
        \frac{\ell (\Theta \mid \mathcal{D})}{\partial \mu_k} = \frac{\partial}{\partial \mu_k} \sum_{i=1}^N \log \left\{ \sum_{l=1}^K \pi_l f (x_i ; \theta_l) \right\} = \sum_{i=1}^N \psi_{ik} (x_i - \mu_k) \\
        \frac{\partial \Delta(\Theta)}{\partial \mu_k} = 2 \left( \frac{1}{N_1} \sum_{i=1}^{N_1} \psi_{i }^{(1)} - \frac{1}{N_2}\sum_{j=1}^{N_2} \psi_{j}^{(2)}\right) \times \left(\frac{1}{N_1} \sum_{i=1}^{N_1}\phi_{ik\tilde{k}}^{(1)} - \frac{1}{N_2} \sum_{j=1}^{N_2} \phi_{jk\tilde{k}}^{(2)}\right)
    \end{gather}
    
    where 
    \begin{align}
        \psi_{ik}^{(s)} & := \psi_k(x_i^{(s)}; \Theta) = \frac{\pi_k f (x_i^{(s)}; \theta_k)}{\sum_{l=1}^K \pi_l f(x_i^{(s)}; \theta_l)} \\
        \phi_{ik \tilde{k}}^{(s)} & = \psi_{i \tilde{k}}^{(s)} (x_i^{(s)} - \mu_{\tilde{k}}) - \psi_{ik}^{(s)} (x_i^{(s)} - \mu_k)
    \end{align}
    Hence, 
    \begin{align}
        \frac{\partial \ell(\Theta; \lambda)}{\partial \mu_k} = \sum_{i=1}^N \psi_{ik}(x_i - \mu_k) - 2\lambda \left(\frac{1}{N_1}\sum_{i=1}^{N_1}\psi_{i \tilde{k}}^{(1)} - 
        \frac{1}{N_2} \sum_{j=1}^{N_2}\psi_{j \tilde{k}}^{(2)}\right) 
        \times
        \left(\frac{1}{N_1} \sum_{i=1}^{N_1} \phi_{ik \tilde{k}}^{(1)} - \frac{1}{N_2} \sum_{j=1}^{N_2} \phi_{jk \tilde{k}}^{(2)} \right)
    \end{align}

\end{itemize}

\begin{itemize}
    \item Gradient with respect to \(\eta_k\)

    \begin{gather}
        \frac{\partial \ell(\Theta; \lambda)}{\partial \eta_k} = \frac{\partial \ell(\Theta \mid \mathcal{D})}{\partial \eta_k} - \lambda \frac{\partial \Delta(\Theta)}{\partial \eta_k}
    \end{gather}
    \begin{align}
        \frac{\ell (\Theta \mid \mathcal{D})}{\partial \eta_k} = \frac{\partial}{\partial \eta_k} \sum_{i=1}^N \log \left\{ \sum_{l=1}^K \pi_l f (x_i ; \theta_l) \right\}
        = \sum_{i=1}^N \sum_{l=1}^K \psi_{il} (\delta_{lk} - \pi_k) = \sum_{i=1}^N (\psi_{ik} - \pi_k)
    \end{align}
    \begin{align}
        \frac{\partial \Delta(\Theta)}{\partial \eta_k} = 2 \left( \frac{1}{N_1} \sum_{i=1}^{N_1} \psi_{i \tilde{k}}^{(1)} - \frac{1}{N_2}\sum_{j=1}^{N_2} \psi_{j \tilde{k}}^{(2)}\right) \times
        \left(\frac{1}{N_1} \left(\delta_{k \tilde{k}} - \psi_{ik}^{(1)}\right) \psi_{i \tilde{k}}^{(1)} - \frac{1}{N_2} \left(\delta_{k \tilde{k}} - \psi_{jk}^{(2)}\right) \psi_{j \tilde{k}}^{(2)} \right)
    \end{align}
    where
    \begin{align}
         \\
        \delta_{k \tilde{k}} & = 
        \begin{cases}
            1 \quad \text{if } k = \tilde{k}\\
            0 \quad \text{if } k \neq \tilde{k}
        \end{cases}
    \end{align}
    Hence,
    \begin{align}
        \frac{\partial \ell(\Theta; \lambda)}{\partial \eta_k} & = \sum_{i=1}^N (\psi_{ik} -\pi_k) 
        \\
        & - 2 \lambda \left(\frac{1}{N_1} \sum_{i=1}^{N_1} \psi_{i \tilde{k}}^{(1)} - \frac{1}{N_2} \sum_{j=1}^{N_2} \psi_{j \tilde{k}}^{(2)} \right) \times
        \left(\frac{1}{N_1} \sum_{i=1}^{N_1} \left(\delta_{k \tilde{k}} - \psi_{ik}^{(1)}\right)\psi_{i \tilde{k}}^{(1)} - \sum_{j=1}^{N_2} \left(\delta_{k \tilde{k}} - \psi_{jk}^{(2)}\right)\psi_{j \tilde{k}}^{(2)}\right).
    \end{align}

\end{itemize}

\clearpage
\section{Implementation Details}

\subsection{Datasets}

\begin{itemize}
    \item \(\textbf{Adult dataset}\): Adult dataset\footnote{\url{http://archive.ics.uci.edu/dataset/2/adult}} is based on US Census Bureau's 1994 dataset.
    The dataset consists of 32,561 samples, with 5 continuous variables (age, fnlwgt, education level, capital-gain, hours-per-week) and 7 categorical variables (workclass, education, occupation, relationship, race, sex, native country). 
    The categorical variables have 9, 16, 7, 15, 6, 5, and 42 unique values, respectively.
    The sensitive attribute is gender (male: 21,790, female: 10,771), resulting in a maximum dataset Balance of 0.4943 (ratio = [0.6692, 0.3308]).
    
    \item \(\textbf{Bank dataset}\): Bank dataset\footnote{\url{https://archive.ics.uci.edu/ml/datasets/Bank+Marketing}} is provided by a Portuguese banking institution and contains 41,008 samples with 21 features. 
    We use 6 continuous variables (age, duration, euribor3m, nr.employed, cons.price.idx, campaign) and 10 categorical features (job, education, default, housing, loan, contact, month, day of week, poutcome, term deposit). 
    The categorical variables have 12, 8, 3, 3, 3, 2, 10, 5, 3, and 2 unique values, respectively. The sensitive attribute is marital status, which has three categories ($M=3$): married, divorced, and single.
    For the analysis with $M=2,$ we combine divorced and single statuses into an "unmarried" category. The ratio of the three sensitive attributes is [0.6064, 0.1122, 0.2814] , and the transformed ratio for the two sensitive attributes is [0.6064, 0.3936].
    Hence, the maximum Balance of Bank dataset is 0.6491 for $M=2$ or 0.1850 for $M=3.$

    \item \(\textbf{Credit dataset}\): Credit dataset\footnote{\url{https://archive.ics.uci.edu/ml/datasets/default+of+credit+card+clients}} is provided by the Department of Information Management, Chung Hua University in Taiwan, and consists of 30,000 samples with 24 features.
    We use 4 continuous variables (LIMIT\_BAL, AGE, BILL\_AMT1, PAY\_AMT1) and 8 categorical variables (SEX, MARRIAGE, PAY\_0, PAY\_2, PAY\_3, PAY\_4, PAY\_5, PAY\_6).
    The categorical variables consist of 2, 4, 11, 11, 11, 11, 10, and 10 unique values for each.
    The sensitive attribute is EDUCATION, which consists of seven categories.
    We transform seven categories into two ($M=2$) or three ($M=3$) status, resulting in [0.4677, 0.5323] or [0.4677, 0.3528, 0.1795].
    The maximum Balance of each case is 0.8785 or 0.3838.

    \item \(\textbf{Census dataset}\): Census dataset\footnote{\url{https://archive.ics.uci.edu/dataset/116/us+census+data+1990}} is a sub-dataset of the 1990 US Census, consisting of 2,458,285 samples with 68 attributes.
    We use 25 continuous variable, similar to the approach taken by \citep{backurs2019scalable} and \citep{ziko2021variational}. 
    The sensitive attribute is gender, with a male/female ratio of [0.4847, 0.5153].
    The dataset has the maximum Balance of 0.9407.
\end{itemize}

\subsection{Multinoulli Mixture Model for Categorical Data}
Denote $d_{\textup{cate}}$ as the number of categorical features.
Let $l_j$ for $j\in[d_{\textup{cate}}]$ be the number of categories in $j$-th feature. 
For notational simplicity, we let $\{1,\ldots,l_j\}$ be the categories of $j$-th feature $x_{ij}$. 
For each components, parameters are represented as $\theta_k=(\phi_{k,1},\ldots,\phi_{k,d_{\textup{cate}}})$. 
Here, $\phi_{k,j}=(p_{k,j,1},\ldots,p_{k,j,l_j})\in[0,1]^{l_j}$ is the vector of the probability masses of categorical distributions on $(l_j-1)$-dimensional simplex, for $k\in[K]$. 
Then, the mixture density for categorical data can be formulated as: 
\begin{equation}\label{eq:multinoulli_mixture}
    \mathbf{f} (x_i \mid \Theta)
    = \sum_{k=1}^K \pi_k \prod_{j=1}^{d_{\textup{cate}}} \textup{Cat}\bigl(x_{ij};\phi_{k,j}\bigr)
    = \sum_{k=1}^K \pi_k \prod_{j=1}^{d_{\textup{cate}}} \prod_{c=1}^{l_j} p_{k,j,c}^{\mathbb{I}(x_{ij}=c)},
\end{equation}
which yields the complete-data log-likelihood of the form:
\begin{equation}\label{eq:multinoulli_mixture_lcomp}
    \ell_{{comp}}(\Theta)
    = \sum_{i=1}^N \sum_{k=1}^K \mathbb{I}(Z_{i} = k) \Bigl[\log\pi_k + \sum_{j=1}^{d_{\textup{cate}}}\sum_{c=1}^{l_j} \mathbb{I} (x_{ij}=c) \cdot \log p_{k,j,c}\Bigr].
\end{equation}
We randomly initialize all parameters $p_{k,j,c}, c \in [l_{j}], j \in [d_{\textup{cate}}], k \in [K].$
That is, we first sample $\tilde{p}_{k,j,c}, c \in [l_{j}], j \in [d_{\textup{cate}}], k \in [K]$ from $\textup{Unif}(0, 1),$ 
then re-parametrize $\tilde{p}$ via softmax function to satisfy $\sum_{c=1}^{l_{j}} p_{k, j, c} = 1$ for all $k \in [K]$ and $j \in [d_{\textup{cate}}].$

\subsection{Hyper-Parameters}

For FMC-EM, we tune the hyper-parameters over the grids $T\in\{100,200,400\}$, $R\in\{5,10,20\}$, and $\gamma\in\{10^{-3},10^{-2},10^{-1}\}$.  
Among these values, we use $(T,R,\gamma)=(200,10,10^{-2})$ for all cases in our experiments, since we confirm that $T=200$ with $R=10$ ensures stable convergence, and $\gamma=10^{-2}$ provides stable yet sufficiently fast parameter updates.

\subsection{Computing Resources}

For all the experiments, we use several Intel(R) Xeon(R) Silver 4410Y CPU cores as well as RTX 3090 GPU processors.

\clearpage
\section{Omitted Experimental Results}

\subsection{Comparing FMC-EM and FMC-GD}

\cref{fig:10_seeds_remaining} below shows the remaining results of Pareto front plots between $\{\Delta$, Balance$\}$ and Cost with standard deviation bands, on Adult, Bank and Credit datasets.
Similar to \cref{fig:10_seeds_del_bal}, it shows that FMC-EM outperforms FMC-GD (in terms of cost-fairness trade-off) with smaller variances.

\begin{figure*}[h]
    \centering
    \includegraphics[width=0.24\linewidth]{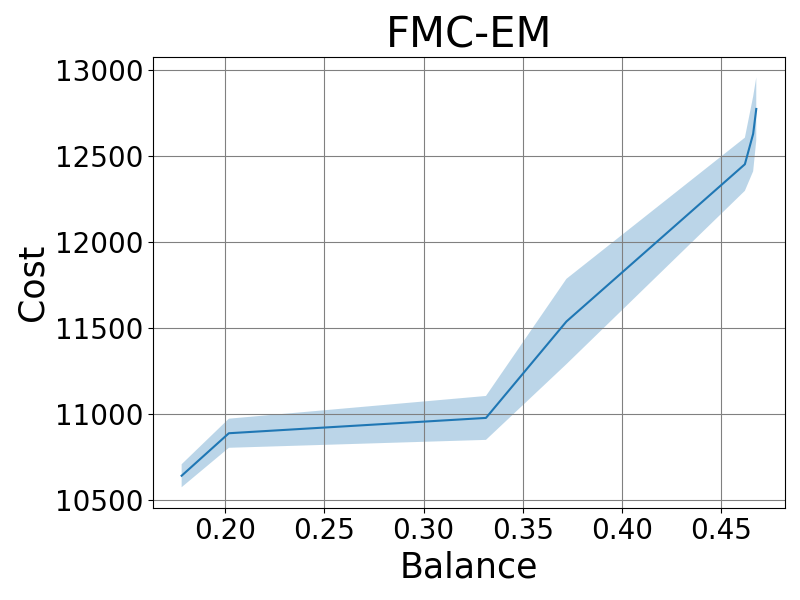}
    \includegraphics[width=0.24\linewidth]{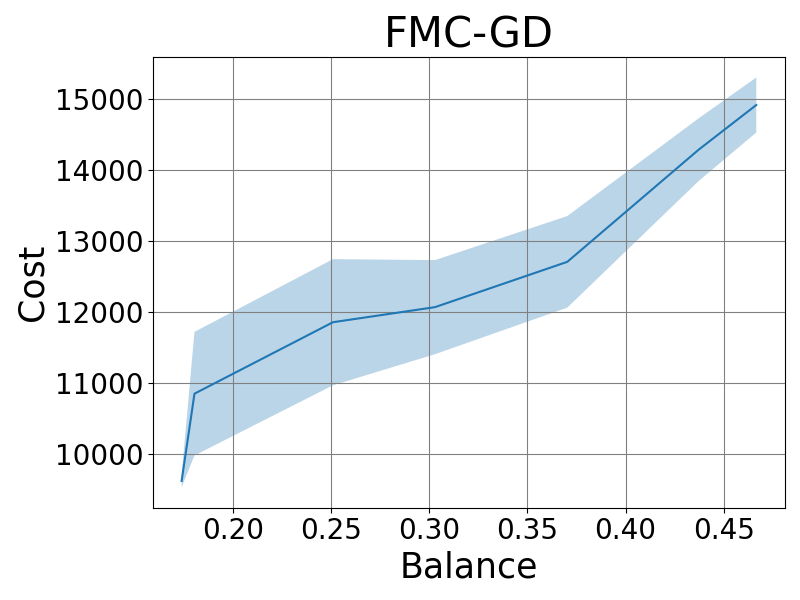}
    \\
    \includegraphics[width=0.24\linewidth]{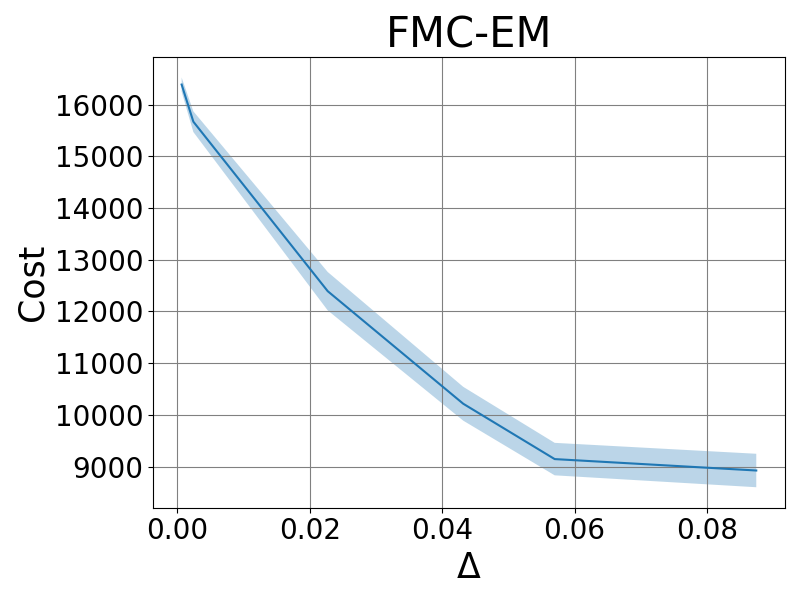}
    \includegraphics[width=0.24\linewidth]{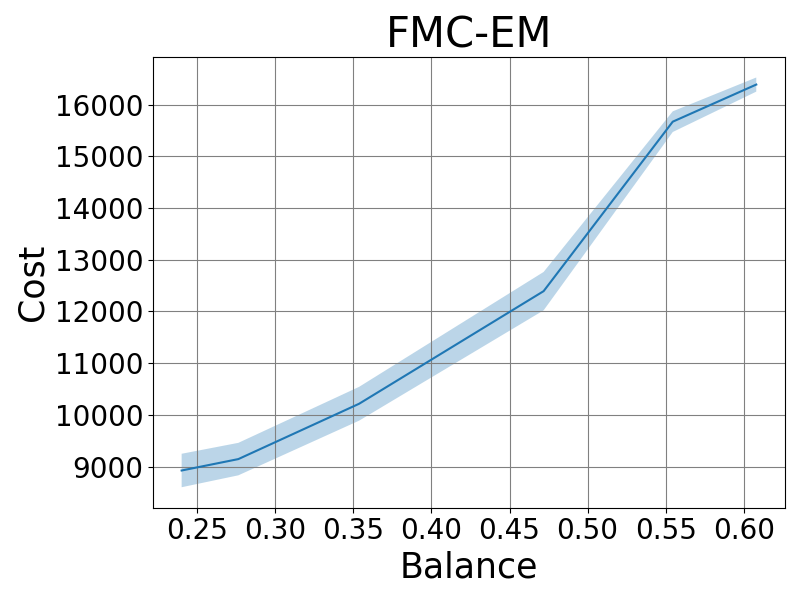}
    \includegraphics[width=0.24\linewidth]{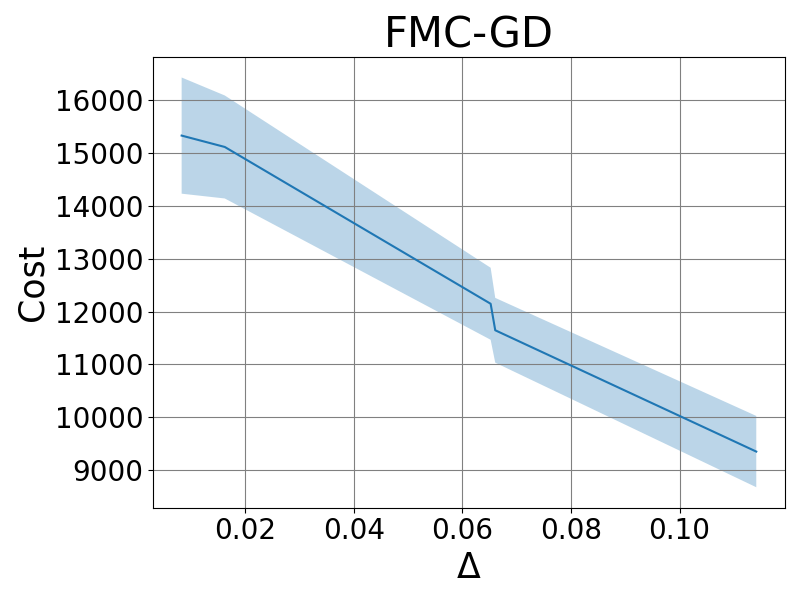}
    \includegraphics[width=0.24\linewidth]{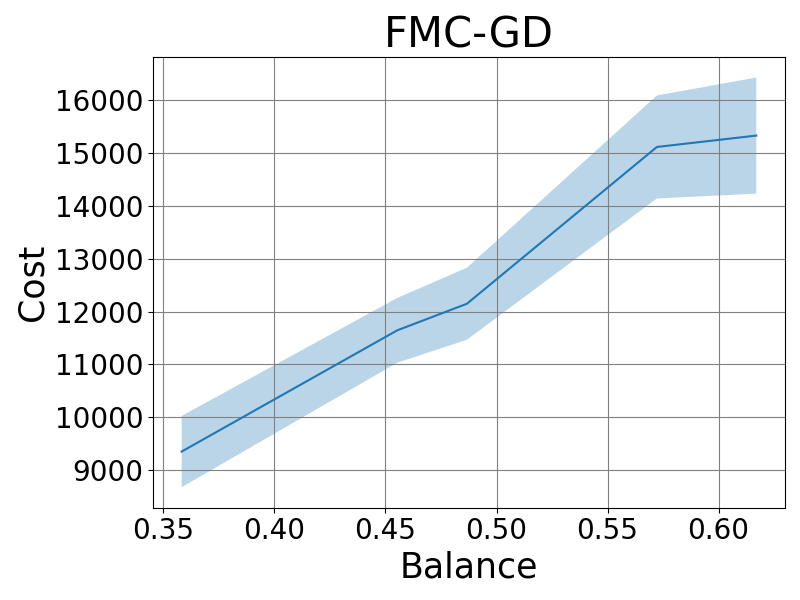}
    \\
    \includegraphics[width=0.24\linewidth]{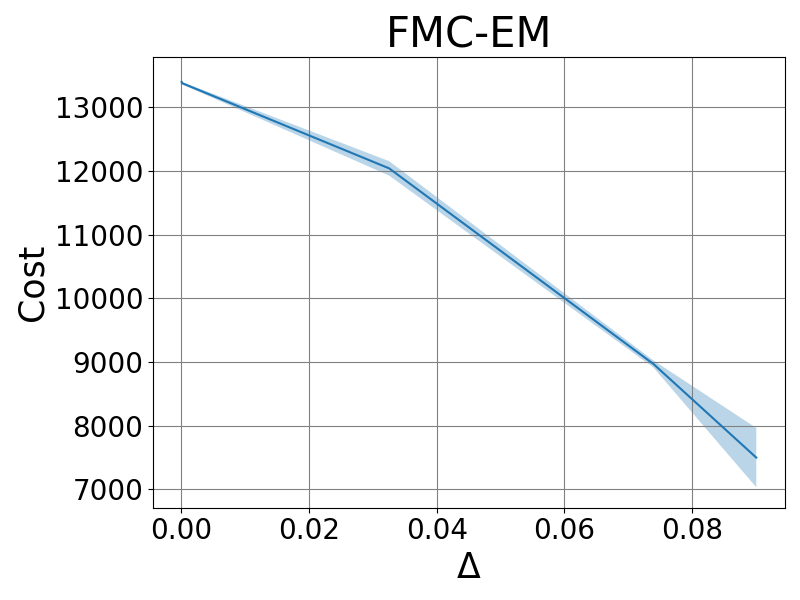}
    \includegraphics[width=0.24\linewidth]{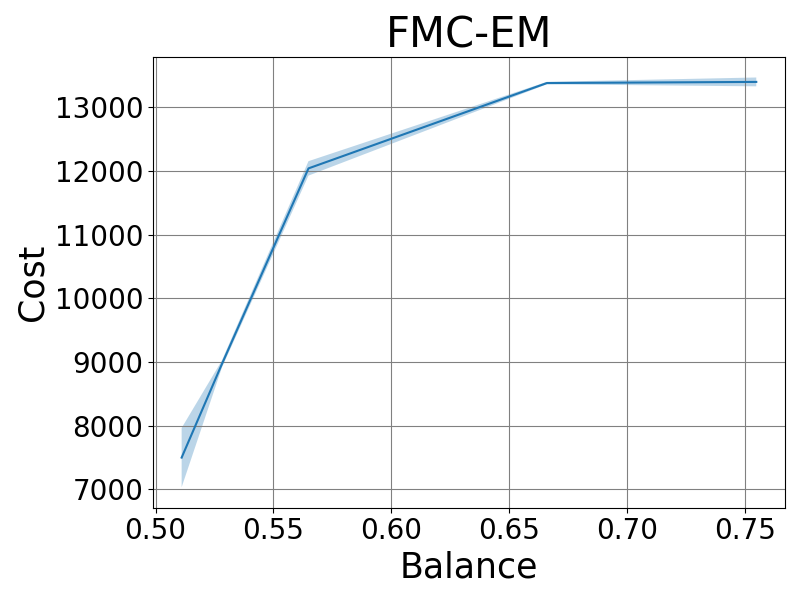}
    \includegraphics[width=0.24\linewidth]{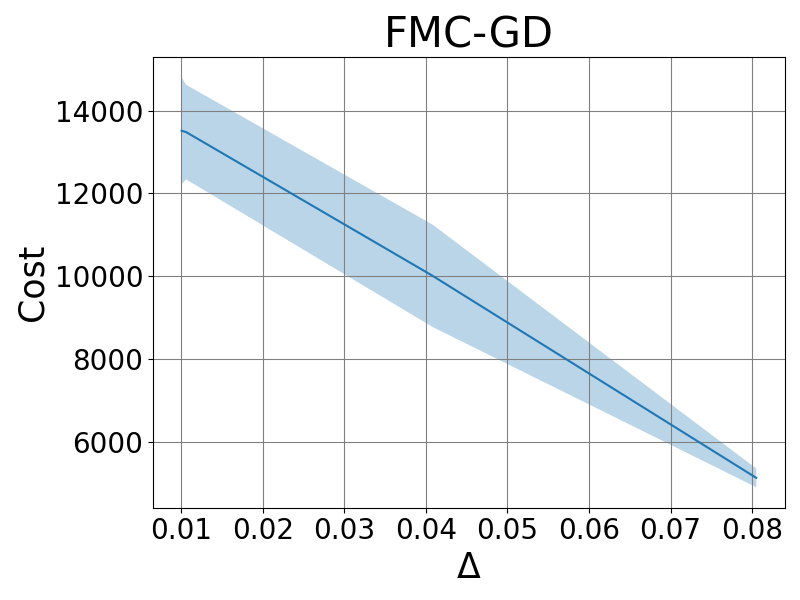}
    \includegraphics[width=0.24\linewidth]{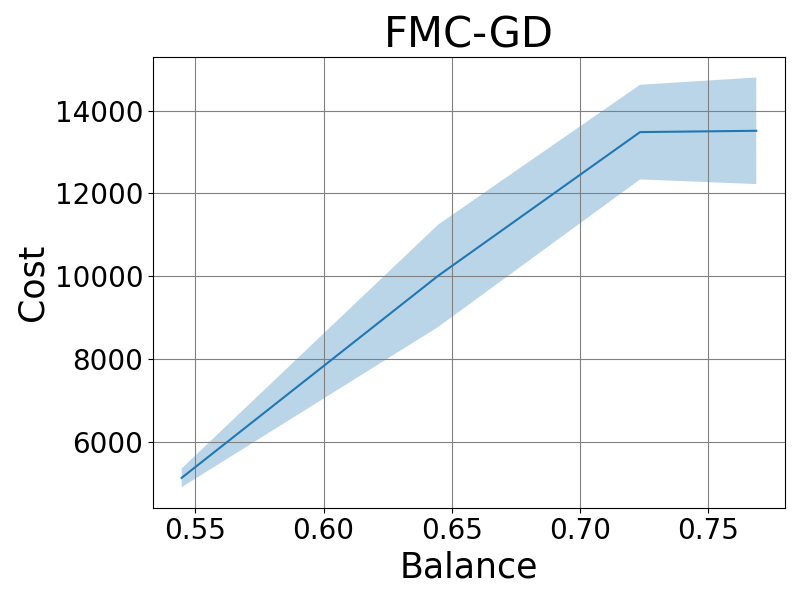}
    \caption{
    First row: Pareto front lines between Balance and Cost with standard deviation bands of FMC-EM and FMC-GD on Adult dataset.
    Second row: Pareto front lines between $\{\Delta$, Balance$\}$ and Cost with standard deviation bands of FMC-EM and FMC-GD on Bank dataset.
    Third row: Pareto front lines between $\{\Delta$, Balance$\}$ and Cost with standard deviation bands of FMC-EM and FMC-GD on Credit dataset.
    }
    \label{fig:10_seeds_remaining}
\end{figure*}

\clearpage
\subsection{Performance Comparison}

\paragraph{FMC vs. Baseline algorithms}

\cref{fig:three_L2_pareto_Balance} below shows the Pareto front plots between Balance and Cost on Adult, Bank, and Credit datasets.
Similar to \cref{fig:three_L2_pareto_delta}, we observe that FMC still outperforms SFC, while VFC is limited to achieving fairness levels under certain value (the end of the blue dashed line).

\begin{figure}[h]
    \centering
    \includegraphics[width=0.25\linewidth]{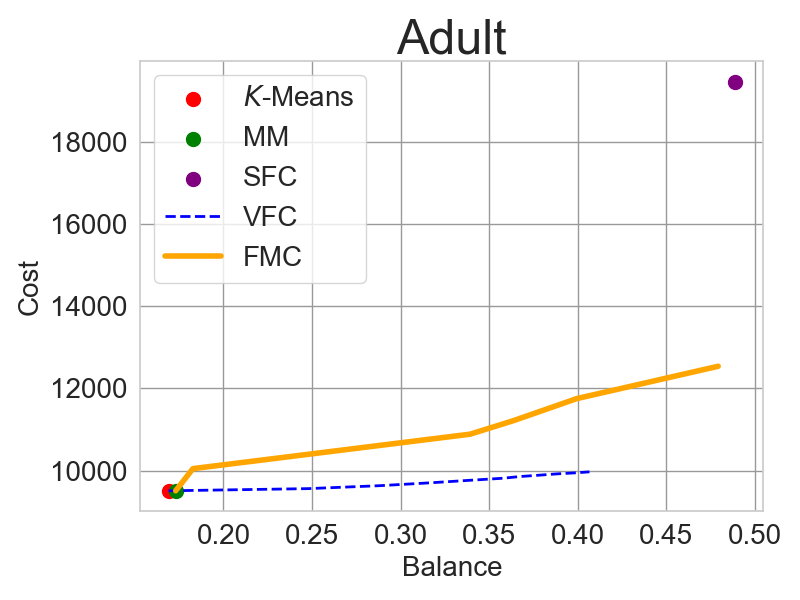}
    \includegraphics[width=0.25\linewidth]{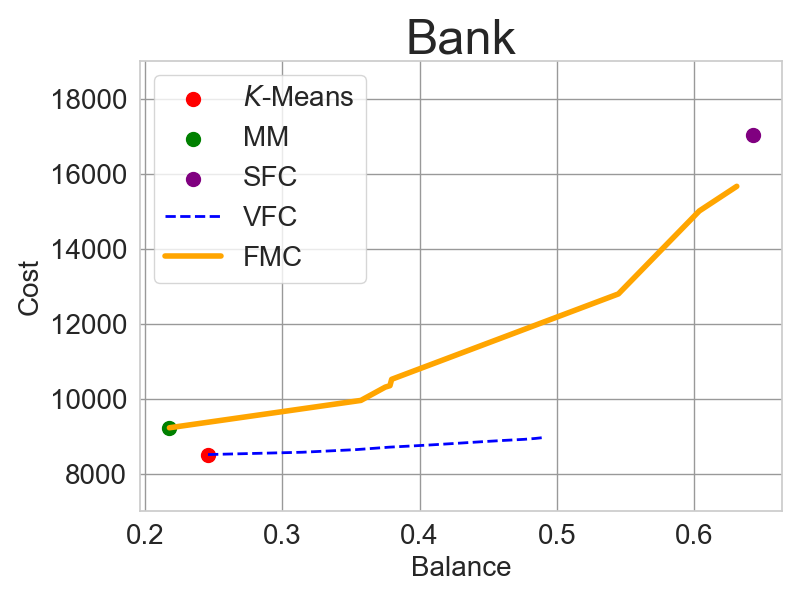}
    \includegraphics[width=0.25\linewidth]{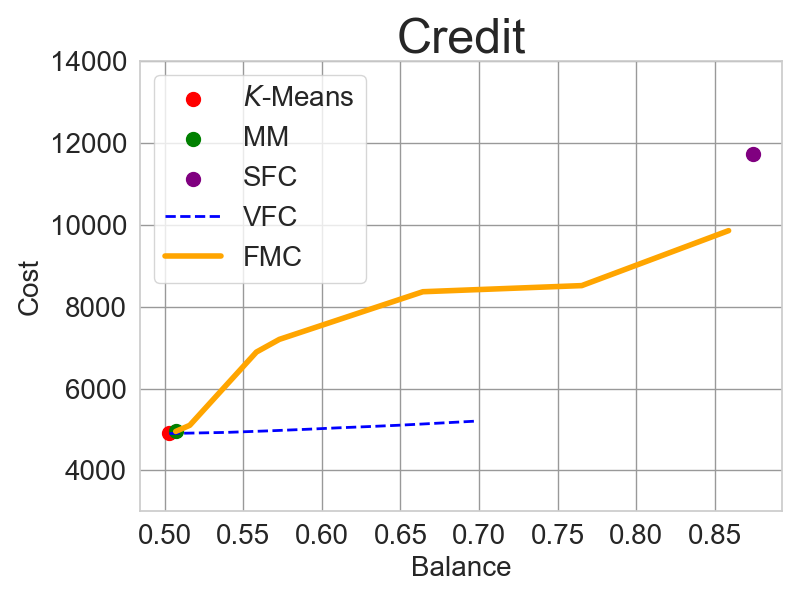}
    \caption{\small Pareto front lines between Balance and Cost on Adult, Bank, and Credit datasets.
    }
    \label{fig:three_L2_pareto_Balance}
\end{figure}

\cref{fig:three_noL2_pareto} shows the Pareto front plots between fairness levels ($\Delta$ and Balance) and Cost without $L_2$ normalization.
Without $L_2$ normalization, VFC may fail to converge (no line for VFC in Credit dataset), which is previously discussed in \citet{kim2025fairclusteringalignment}.

\begin{figure}[h]
    \centering
    \includegraphics[width=0.25\linewidth]{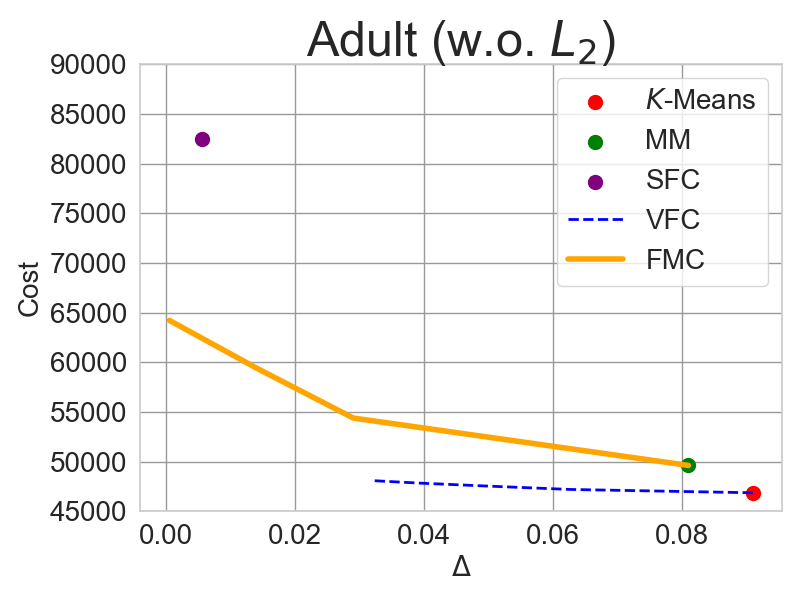}
    \includegraphics[width=0.25\linewidth]{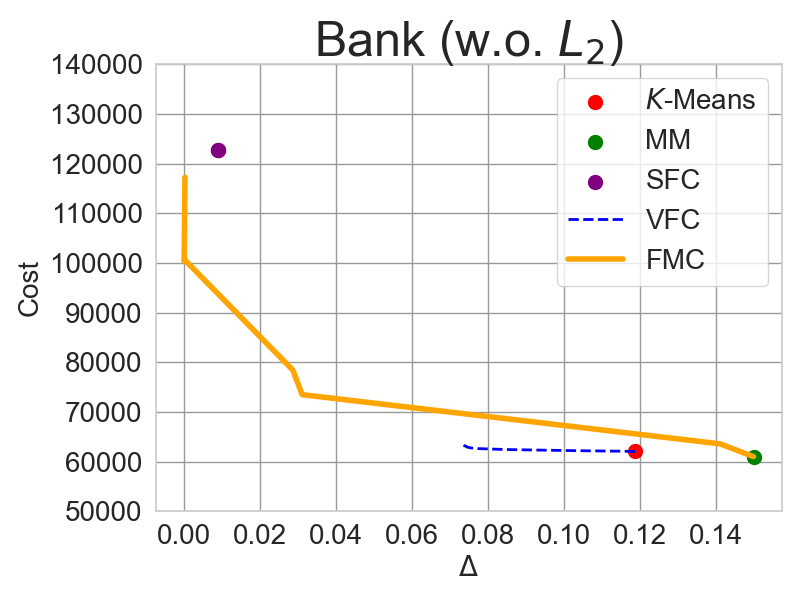}
    \includegraphics[width=0.25\linewidth]{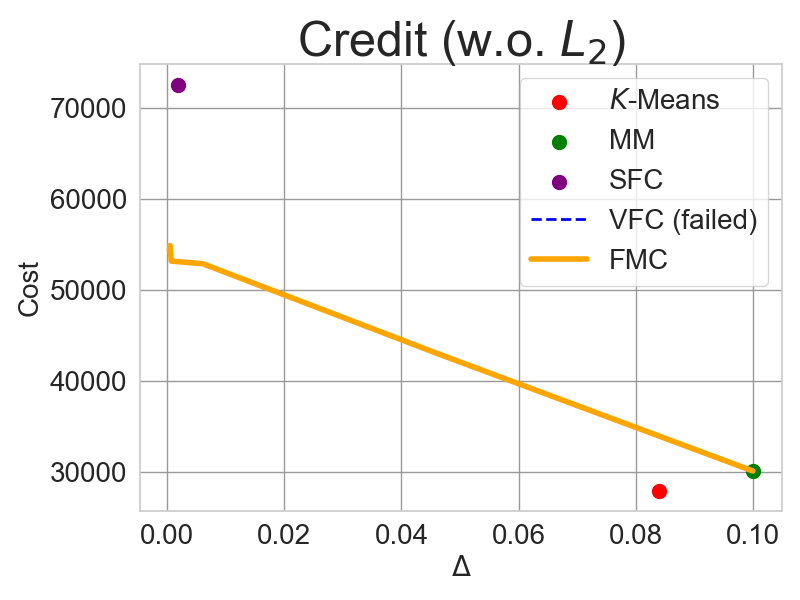}
    \includegraphics[width=0.25\linewidth]{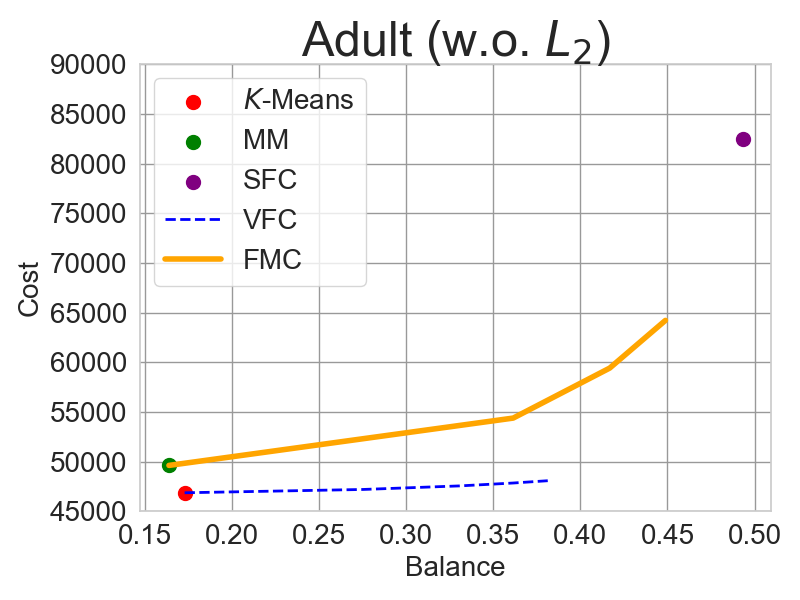}
    \includegraphics[width=0.25\linewidth]{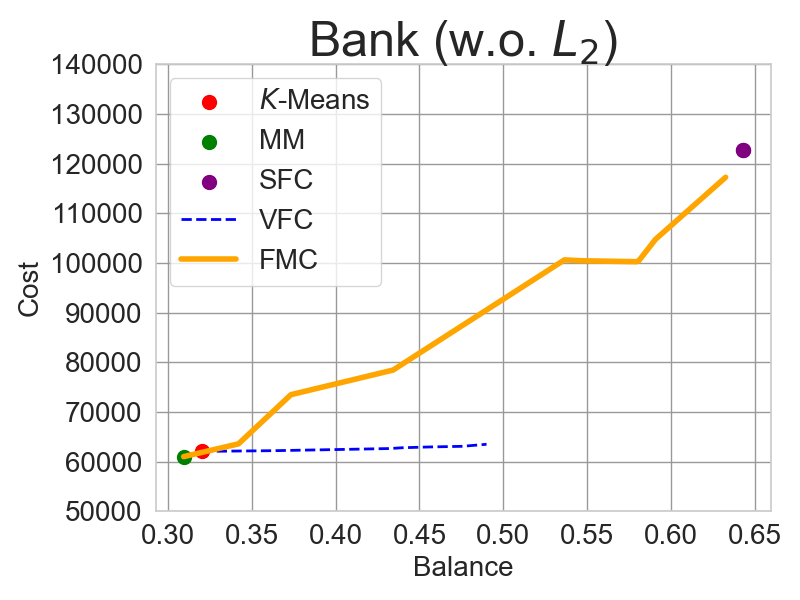}
    \includegraphics[width=0.25\linewidth]{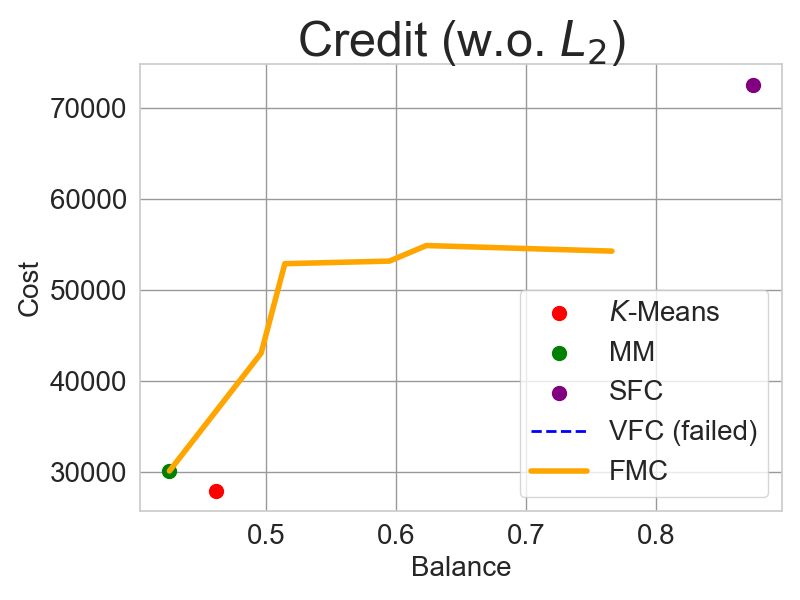}
    \caption{\small 
    (Top three) Pareto front lines between $\Delta$ and Cost on Adult, Bank, and Credit datasets, without $L_2$ normalization.\\
    (Bottom three) Pareto front lines between Balance and Cost on Adult, Bank, and Credit datasets, without $L_2$ normalization.
    }
    \label{fig:three_noL2_pareto}
\end{figure}

Additionally, \cref{table:compare_cate_add_FCA} compares FMC and FCA \citep{kim2025fairclusteringalignment} on Adult and Bank datasets, in terms of Balance (higher is better) and Cost (lower is better).
The results suggest that FMC slightly underperforms FCA, but by a moderate margin, even though FCA is a deterministic method that directly optimizes the clustering cost under a fairness constraint rather than relying on model-based clustering ad FMC does.

\begin{table}[h]
  \centering
  \footnotesize
  \begin{tabular}{cccccc}
    \toprule
    \multicolumn{2}{c}{Dataset} & \multicolumn{2}{c}{Adult} & \multicolumn{2}{c}{Bank}\\
    \midrule
    \multicolumn{2}{c}{Algorithm} & Balance \(\uparrow\) & Cost \(\downarrow\) & Balance \(\uparrow\) & Cost \(\downarrow\) \\
    \cmidrule(r){0-1}
    \cmidrule(lr){3-4}
    \cmidrule(l){5-6}
    \multirow{3}{*}{Standard clustering} 
    & $K$-means    &   0.169   &   9509   &   0.246   &   8513 \\
    & MM-GD     &   0.174   &   9629   &   0.243   &   8603 \\
    & MM-EM     &   0.173   &   9636   &   0.242   &   8600 \\
    \arrayrulecolor{gray!80}\midrule\arrayrulecolor{black}
    \multirow{3}{*}{Fair clustering} 
    & FCA       &   0.494        & 10680  &   0.645        &  10853   \\
    & FMC-GD    &   0.467   &   14921   &   0.635   &   15725   \\
    & FMC-EM    &   0.481   &   12715   &   0.629   &   16635   \\
    \bottomrule
  \end{tabular}
  \caption{Comparison of FMC and FCA in terms of Balance and Cost on Adult and Bank datasets.}
  \label{table:compare_cate_add_FCA}
\end{table}

\clearpage
\subsection{Analysis on Large-Scale Dataset}

\paragraph{FMC vs. Baseline algorithms}

\cref{fig:census_noL2_pareto} shows the Pareto front trade-off between fairness levels and Cost, without $L_2$ normalization.
Similar to \cref{fig:census_L2_pareto_delta}, FMC outperforms SFC in that FMC achieves lower Cost than SFC under a perfect fairness level (i.e., $\Delta \approx 0.0$).
Note that, without $L_2$ normalization, there is no line for VFC since VFC fails to converge, which is previously discussed in \citet{kim2025fairclusteringalignment}.

\begin{figure}[h]
    \centering
    \includegraphics[width=0.3\linewidth]{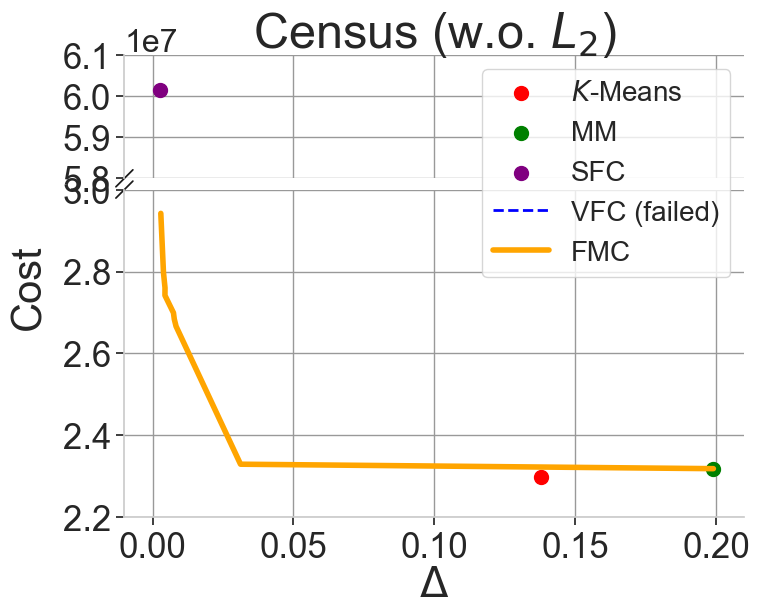}
    \includegraphics[width=0.3\linewidth]{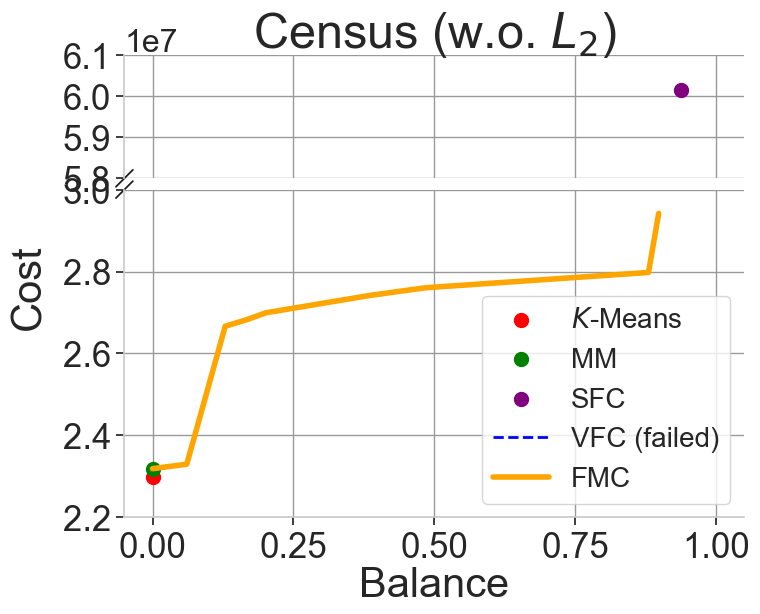}
    \caption{
    (Left) Pareto front lines between $\Delta$ and Cost on Census dataset, without $L_2$ normalization.
    (Right) Pareto front lines between Balance and Cost on Census dataset, without $L_2$ normalization.
    }
    \label{fig:census_noL2_pareto}
\end{figure}

\paragraph{Mini-batch learning vs. Sub-sample learning}

Without $L_2$ normalization, we also examine whether sub-sample learning can empirically perform similar to mini-batch learning by varying the sub-sample size among $\{1\%, 3\%, 5\%, 10\%\}$.
Similar to \cref{table:large_data_subsampling_L2} (i.e., results with $L_2$ normalization), the results in \cref{table:large_data_subsampling} demonstrate that sub-sample learning with sizes greater than 5\% achieves similar empirical performance with that of mini-batch learning, while requiring fewer computations.
These results also suggest that sub-sample learning can be particularly advantageous when the dataset is very large.

\begin{table}[h]
  \centering
  \footnotesize
  \begin{tabular}{l|cccc}
    \toprule
    Method $\left( \frac{n}{N} \right)$ & Cost ($\times 10^{7}$) & $\Delta$ & Balance & Time
    \\
    \midrule
    SS (1\%) & 3.073 & 0.011 & 0.832 & 52.4\%
    \\
    SS (3\%) & 2.907 & 0.006 & 0.895 & 54.5\%
    \\
    SS (5\%) & 3.096 & 0.003 & 0.898 & 58.5\%
    \\
    SS (10\%) & 3.049 & 0.003 & 0.902 & 79.2\%
    \\
    \midrule
    MB (10\%) & 3.043 & 0.003 & 0.896 & 100.0\%
    \\
    \bottomrule
  \end{tabular}
  \caption{Performances of mini-batch learning (MB) and sub-sample learning (SS) on Census dataset, without $L_2$ normalization.
  The computation times of SS are measured by the relative ratio compared to the mini-batch learning with $10\%$, and the Lagrangian $\lambda$ for each case is tuned to achieve the lowest Gap value i.e., $\Delta \approx 0.$}
  \label{table:large_data_subsampling}
\end{table}

\subsection{Categorical Data Analysis}

\cref{table:compare_cate_m=2,table:compare_cate_m=3} confirm the adaptability of FMC to categorical data with two and three sensitive groups, respectively.
The results indicate that FMC-GD and FMC-EM can produce fair clusterings for categorical data well.
Here, MM-GD refers to the mixture model learned by gradient‐descent optimization, while MM-EM refers to the mixture model learned by the EM algorithm.

\begin{table*}[h]
  \centering
  \footnotesize
  \begin{tabular}{ccccccccccc}
    \toprule
    \multicolumn{2}{c}{Dataset} & \multicolumn{3}{c}{Adult} & \multicolumn{3}{c}{Bank} & \multicolumn{3}{c}{Credit}\\
    \midrule
    \multicolumn{2}{c}{Algorithm} & Balance \(\uparrow\) & NLL \(\downarrow\) & $\Delta$ \(\downarrow\) & Balance \(\uparrow\) & NLL \(\downarrow\) & $\Delta$ \(\downarrow\) & Balance \(\uparrow\) & NLL \(\downarrow\) & $\Delta$ \(\downarrow\)\\
    \cmidrule(r){0-1}
    \cmidrule(lr){3-5}
    \cmidrule(l){6-8}
    \cmidrule(l){9-11}
    \multirow{2}{*}{Standard clustering} & MM-GD   &   0.004   &   8.609   &   0.089   &   0.322   &   9.706   &   0.037   &   0.534   &   5.749   &   0.030   \\
    & MM-EM   &   0.004   &   8.618   &   0.099   &   0.312   &   9.771   &   0.030   &   0.534   &   5.751   &   0.029   \\
    \arrayrulecolor{gray!80}\midrule\arrayrulecolor{black}
    \multirow{2}{*}{Fair clustering} & FMC-GD  &   0.467   &   9.524   &   0.008   &   0.631   &   9.919   &   0.002   &   0.732   &   8.232   &   0.005   \\
    & FMC-EM  &   0.480   &   11.600  &   0.001   &   0.633   &   9.940   &   0.002   &   0.739   &   7.706   &   0.001   \\
    \bottomrule
  \end{tabular}
  \normalsize
  \caption{Clustering results (i.e., Balance, NLL, and $\Delta$) on three datasets with categorical data.}
  \label{table:compare_cate_m=2}
\end{table*}

\begin{table*}[h]
  \centering
  \footnotesize
  \begin{tabular}{cccccccc}
    \toprule
    \multicolumn{2}{c}{Dataset} & \multicolumn{3}{c}{Bank} & \multicolumn{3}{c}{Credit}\\
    \midrule
    \multicolumn{2}{c}{Algorithm} & Balance \(\uparrow\) & NLL \(\downarrow\) & $\Delta$ \(\downarrow\) & Balance \(\uparrow\) & NLL \(\downarrow\) & $\Delta$ \(\downarrow\)\\
    \cmidrule(r){0-1}
    \cmidrule(lr){3-5}
    \cmidrule(l){6-8}
    \multirow{2}{*}{Standard clustering}
    & MM-GD    & 0.119 & 9.700 & 0.027 & 0.286 & 5.762 & 0.029 \\
    & MM-EM    & 0.114 & 9.717 & 0.027 & 0.298 & 5.791 & 0.028 \\
    \arrayrulecolor{gray!80}\midrule\arrayrulecolor{black}
    \multirow{2}{*}{Fair clustering}
    & FMC-GD & 0.180 & 10.026 & 0.002 & 0.371 & 8.474 & 0.003 \\
    & FMC-EM & 0.179 & 10.281 & 0.002 & 0.373 & 7.617 & 0.006 \\
    \bottomrule
  \end{tabular}
  \normalsize
  \caption{Clustering results (i.e., Balance, NLL, and $\Delta$) on two datasets with categorical data for $M = 3$.}
  \label{table:compare_cate_m=3}
\end{table*}

In \cref{table:cont+cate}, we further present a more concrete and explicit advantage of FMC in jointly handling categorical data alongside continuous data within a single mixture model.

\subsection{Extension to Multinary Sensitive Attributes}

\cref{table:combined_m3} compares two fairness-agnostic methods ($K$-means and MM-EM) with two fair clustering methods (VFC and FMC-EM) on Bank and Credit datasets with three sensitive groups (see \textit{Implementation details} section for the details).
It shows that, not only for binary sensitive groups, but also for three sensitive groups, FMC can achieve near-perfect fairness level and control the fairness level, while VFC shows a lower Cost since VFC minimizes Cost directly while FMC maximizes the log-likelihood, which is different from Cost.

\begin{table*}[h]
  \centering
  \footnotesize
  \begin{tabular}{@{}cccccccc@{}}
    \toprule
    \multicolumn{2}{c}{Dataset} 
    & \multicolumn{3}{c}{Bank} 
    & \multicolumn{3}{c}{Credit} \\
    \midrule
    \multicolumn{2}{c}{Algorithm} & Balance \(\uparrow\) & Cost \(\downarrow\) & \(\Delta \downarrow\)
    & Balance \(\uparrow\) & NLL \(\downarrow\) & \(\Delta \downarrow\) \\
    \cmidrule(r){0-1}
    \cmidrule(lr){3-5}
    \cmidrule(l){6-8}
    \multirow{2}{*}{Standard clustering}
    & $K$-means & 0.087 & 8513  & 0.055 & 0.159 & 4901  & 0.056 \\
    & MM-EM     & 0.085 & 8538  & 0.108 & 0.151 & 4954  & 0.110 \\
    \arrayrulecolor{gray!80}\midrule\arrayrulecolor{black}
    \multirow{3}{*}{Fair clustering} 
    & VFC                   & 0.171 & 9373  & 0.016 & 0.294 & 6058  & 0.019 \\
    & FMC-EM ($\lambda = 4$) & 0.168 & 19059 & 0.013 & 0.295 & 22047 & 0.017 \\
    & FMC-EM ($\lambda = 10$) & 0.180 & 20040 & 0.000 & 0.375 & 23067 & 0.000 \\
    \bottomrule
  \end{tabular}
  \normalsize
  \caption{
  Clustering results (i.e., Balance, Cost, and $\Delta$) on Bank and Credit datasets with continuous data for $M = 3$.
  We set $\lambda = 4$ to compare with VFC (e.g., $\Delta \approx 0.016$ for Bank) and $\lambda = 10.0$ to achieve the perfect fairness (i.e., $\Delta \approx 0$).
  }
  \label{table:combined_m3}
\end{table*}

\clearpage

\subsection{Parameter Impact Study: The Number of Clusters $K$}

\cref{fig:ablation_num_k} illustrates the effect of $K$ on NLL (negative log‐likelihood), $\Delta$, and Balance. 
The solid orange lines denote performances of FMC-EM and the dotted gray lines denote performances of MM-EM. 
As expected, in FMC‐EM, increasing $K$ reduces NLL while $\Delta$ and Balance remain at the perfect fairness level ($\Delta \approx 0$ and Balance $\approx 0.48$).
This result implies that FMC-EM controls the fairness level well, regardless of $K.$
Note that a slight decrease in Balance for FMC-EM is also expected: because we should split an integer‐sized dataset into $K$ partitions, the Balance can drop slightly as $K$ increases.
In contrast, for MM‐EM, NLL, $\Delta$ and Balance are all affected by $K.$

\begin{figure*}[h]
    \centering
    \includegraphics[width=0.95\linewidth]{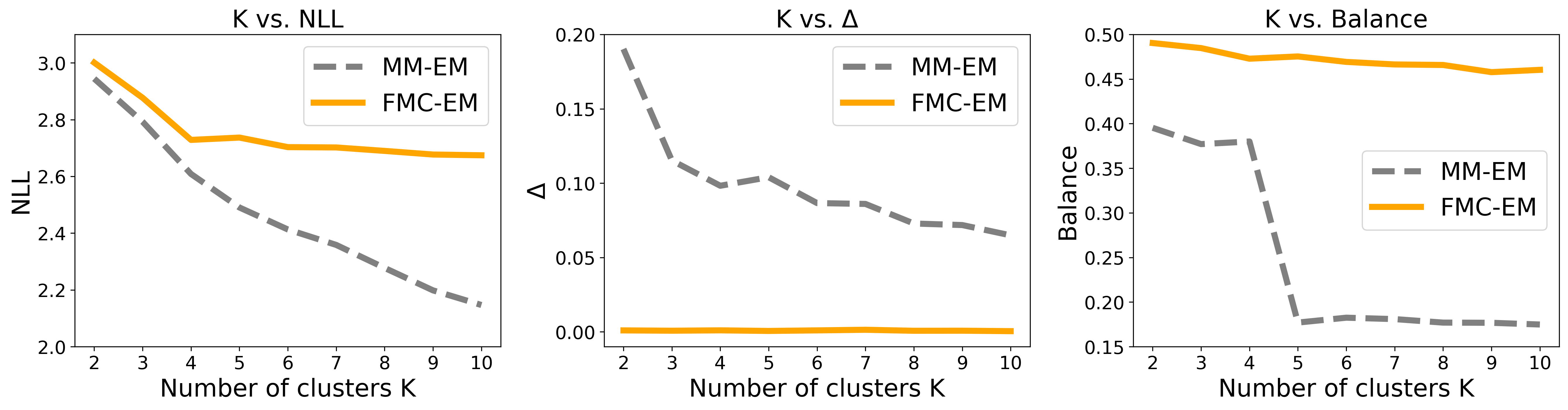}
    \caption{
    The effects of $K$ on NLL, $\Delta,$ and Balance (from left to right) for FMC-EM (solid orange line) and MM-EM (dotted gray line) on Adult dataset.
    The number of clusters $K$ ranges from 2 to 10.
    }
    \label{fig:ablation_num_k}
\end{figure*}

\subsection{Parameter Impact Study: Choice of the Structure of Covariance Matrix}

\cref{fig:All_cov_compare} below shows that the diagonal structure yields substantially lower negative log-likelihood than the isotropic structure, under similar levels of Balance.
This result suggests that, the diagonal structure is generally helpful in view of log-likelihood, which is an additional advantage of our model-based clustering approach.

\begin{figure}[h]
    \centering
    \includegraphics[width=0.33\linewidth]{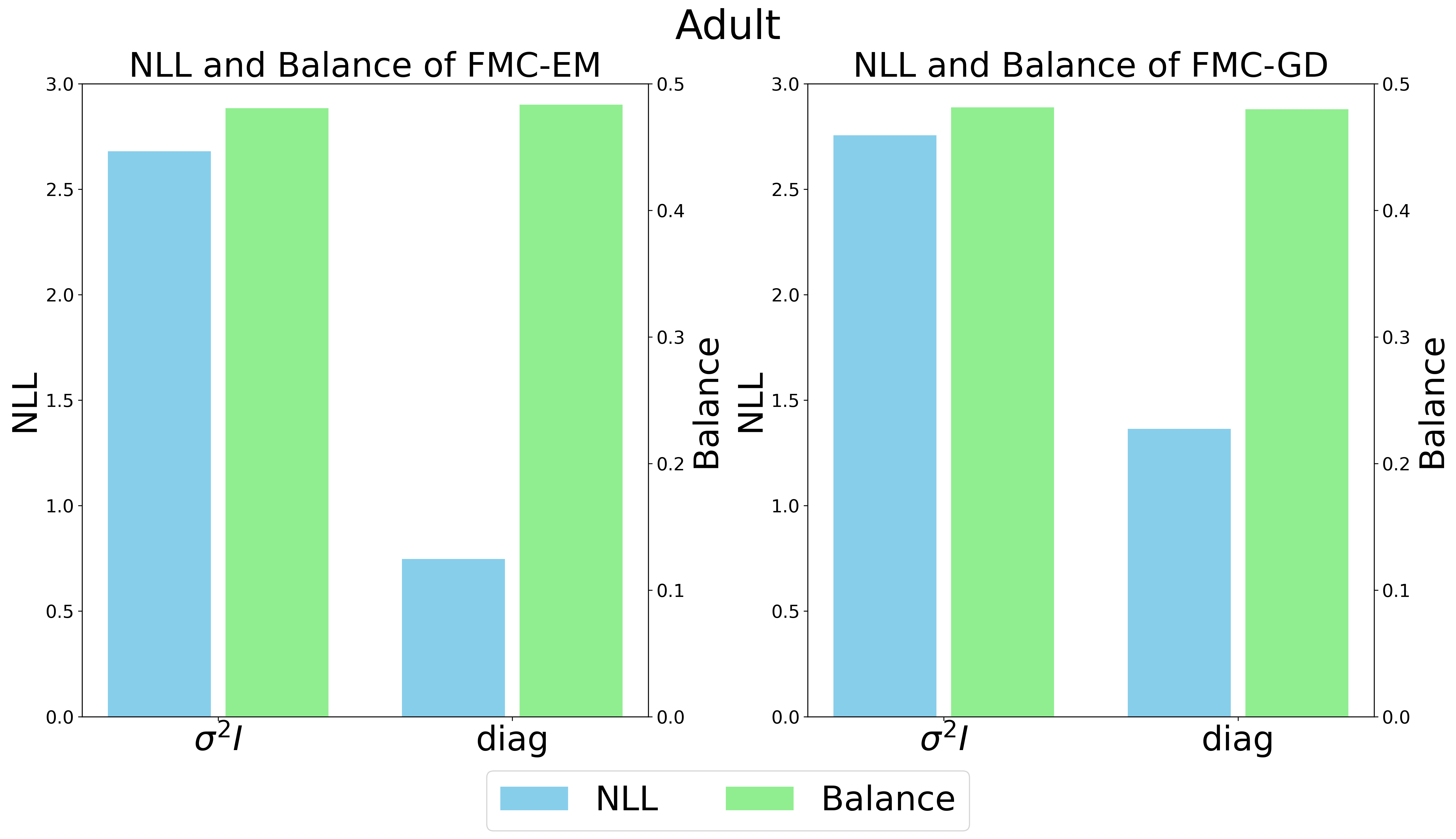}
    \includegraphics[width=0.33\linewidth]{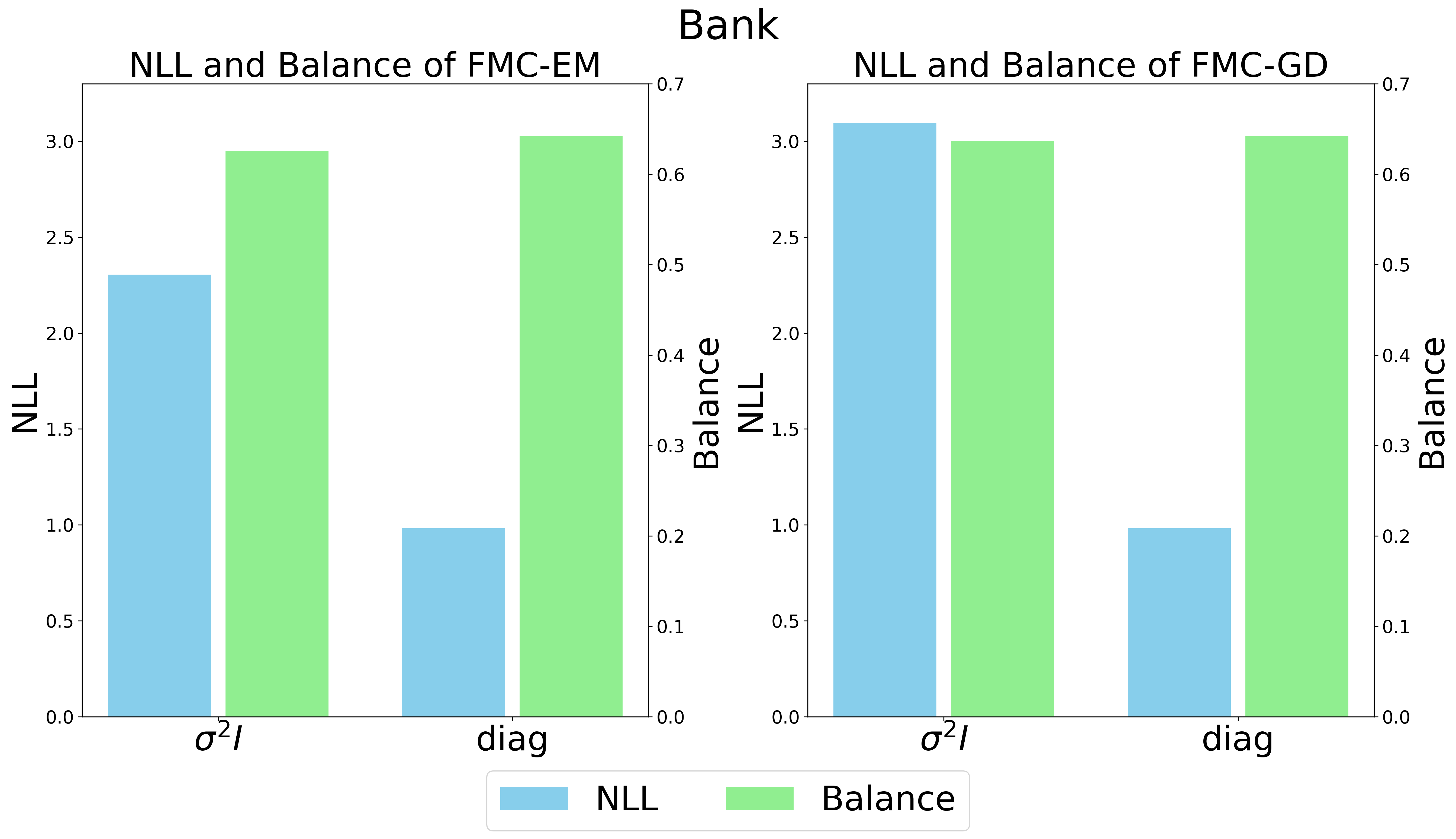}
    \includegraphics[width=0.33\linewidth]{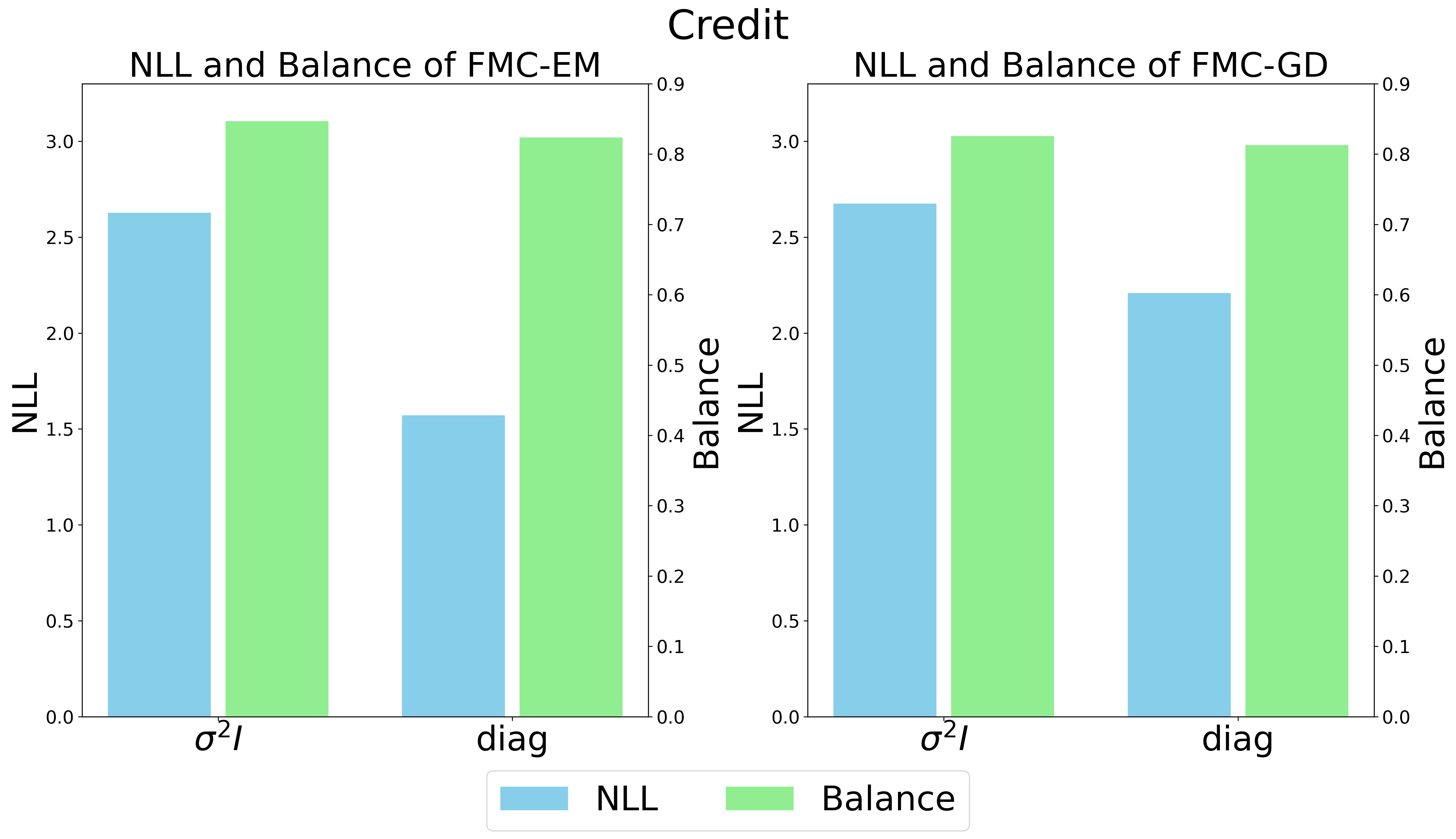}
    \caption{
    Comparison of different covariance structures (isotropic vs. diagonal) in FMC-EM and FMC-GD on Adult, Bank, and Credit datasets, in terms of the negative log-likelihood (NLL) and Balance. 
    $\sigma^2 I$ denotes the results from isotropic covariance and diag denotes the results from diagonal covariance. 
    }
    \label{fig:All_cov_compare}
\end{figure}

\clearpage
\subsection{Empirical convergence of FMC-EM}
\subsection{Empirical Convergence of FMC-EM}

To confirm the empirical convergence of FMC-EM, we track NLL (negative log-likelihood), Cost, $\Delta,$ and Balance, over 200 iterations on Adult dataset.
We set $\lambda = 5.0$ to achieve perfect fairness (i.e., $\Delta \approx 0$).
The results in \cref{fig:variations} shows that FMC-EM converges well in terms of the four measures in practice.
Furthermore, we observe that FMC-EM first improves fairness (i.e., decreases $\Delta$ and increases Balance) in the early stages, and then subsequently reduces NLL and Cost.

\begin{figure*}[h]
    \centering
    \includegraphics[width=0.24\linewidth]{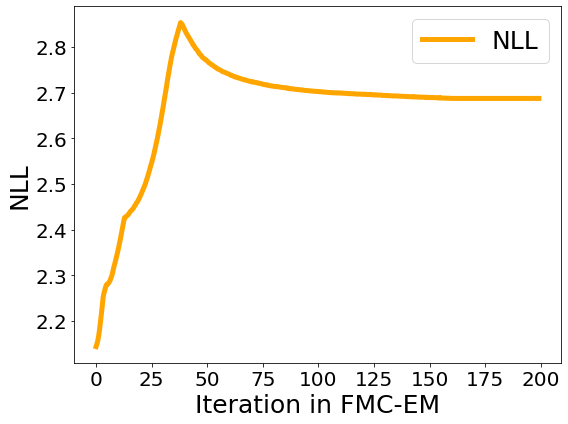}
    \includegraphics[width=0.24\linewidth]{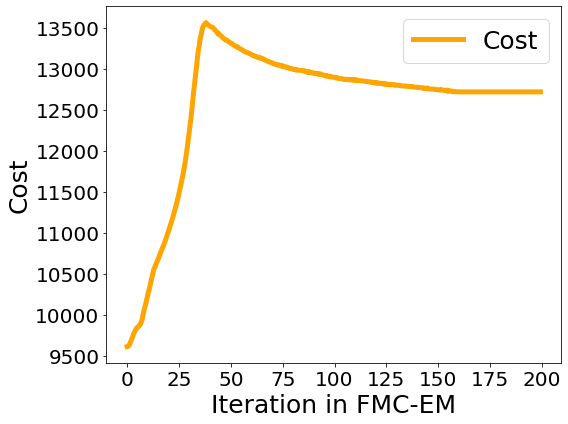}
    \includegraphics[width=0.24\linewidth]{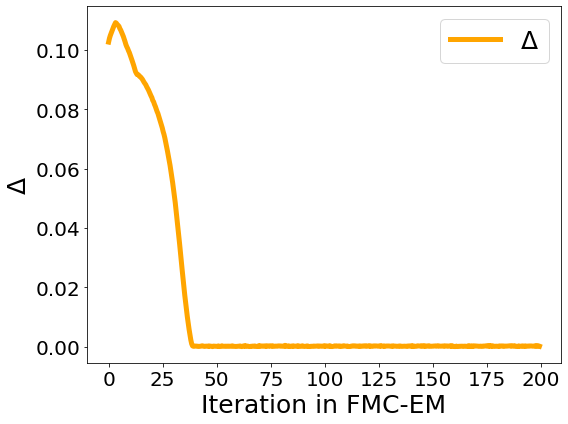}
    \includegraphics[width=0.24\linewidth]{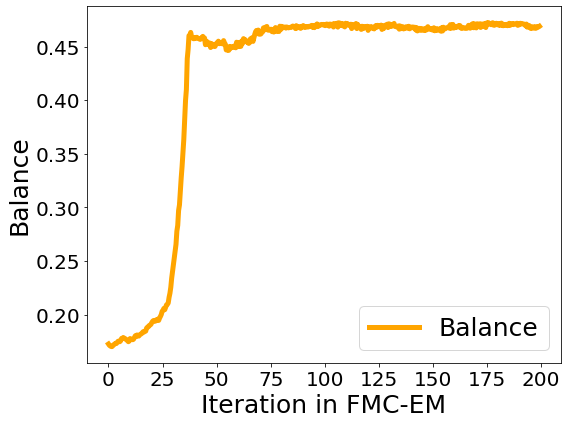}
    \caption{Convergence of FMC-EM (with $\lambda = 5.0$) in terms of NLL, Cost, $\Delta,$ and Balance on Adult dataset.}
    \label{fig:variations}
\end{figure*}

\subsection{Empirical Relationship Between $\Delta$ and Balance}

Here, we empirically verify the relationship between $\Delta$ and Balance (i.e., we check whether a smaller $\Delta$ corresponds to a larger Balance).
\cref{fig:delta_balance} below plots $\Delta$ against Balance of FMC-EM for various $\lambda$ values on Adult dataset.
It implies that, smaller $\Delta$ tends to correspond to higher Balance in practice, demonstrating that reducing $\Delta$ can contribute to increase Balance.
Therefore, using $\Delta$ as a proxy for Balance during the training phase incurs small discrepancy, in other words, we can effectively control Balance by controlling $\Delta.$

\begin{figure*}[h]
    \centering
    \includegraphics[width=0.4\linewidth]{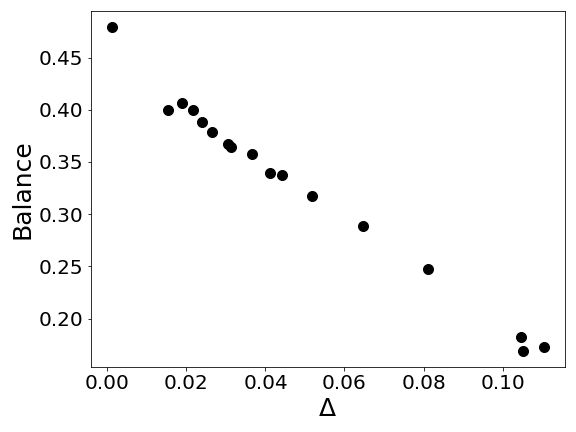}
    \caption{$\Delta$ vs. Balance on Adult dataset.
    The points are obtained from FMC-EM with various values of $\lambda.$}
    \label{fig:delta_balance}
\end{figure*}

\clearpage
\subsection{Advantage of FMC: Joint Modeling for Continuous and Categorical Data}

FMC can be applied to any variable type provided that the likelihood is well-defined.
For example, FMC can handle both continuous and categorical data using a single mixture model.  
By modeling heterogeneous variable types simultaneously, we here show that FMC achieves superior clustering performance (i.e., outperforms both baseline methods and FMC applied to continuous features only) in terms of classification accuracy.

\paragraph{Method}
Assume that given data has two types of variables: continuous and categorical.
We divide the variables of each data point into continuous and categorical, i.e., $x_{i} = (x_{i}^\textup{cont}, x_{i}^\textup{cate}), \forall i \in [N].$ 
Let $x_{i}^\textup{cont}=(x_{i1}^\textup{cont},\ldots,x_{id_\textup{cont}}^\textup{cont})\in\mathbb{R}^{d_\textup{cont}}$ and $x_{i}^\textup{cate}=(x_{i1}^\textup{cate},\ldots,x_{id_\textup{cate}}^\textup{cate})\in\mathbb{Z}^{d_\textup{cate}}$. 
For the continuous part $x_{i}^\textup{cont},$ we consider the Gaussian mixture model with mean vectors $\boldsymbol{\mu} = (\mu_{1}, \ldots, \mu_{K})$ with an isotropic covariance structure $\sigma^{2} \mathbb{I}_{d_\textup{cont}}$, as done in our main analysis.
For the categorical part $x_{i}^\textup{cate},$ we consider the multinoulli mixture model, described in \textit{Multinouilli mixture model for categorical data} section.
Recall that parameters for multinoulli mixture are $\boldsymbol{\phi} = \{(\phi_{k,1},\ldots,\phi_{k,d_{\textup{cate}}})\}_{k=1}^{K},$ where $\phi_{k,j}=(p_{k,j,1},\ldots,p_{k,j,l_j}) \in [0,1]^{l_j}$ is the vector of the probability masses of categorical distributions for $k\in[K]$.

Suppose that the continuous and categorical data are independent.
Then, the mixture density for the `mixed' (= continuous + categorical) variables can be represented as:
\begin{equation}
    \mathbf{f}(x_i \mid \Theta)=\sum_{k=1}^{K}\pi_k \mathcal{N}(x_{i}^\textup{cont};\mu_k,\sigma^2\mathbb{I}_{d_\textup{cont}})\prod_{j=1}^{d_\textup{cate}}\textup{Cat}(x_{ij}^\textup{cate};\phi_{k,j}),
\end{equation}
where $\Theta = (\boldsymbol{\pi},(\boldsymbol{\mu}, \sigma, \boldsymbol{\phi})),$
leading to the complete-data log-likelihood whose form is:
\begin{equation}\label{eq:lcomp_mixed}
    \ell_{comp}^\textup{mixed}(\Theta \mid Y) = \sum_{i=1}^N \sum_{k=1}^K \mathbb{I}(Z_{i} = k) 
    \left(
    \log \pi_{k} - \frac{d_\textup{cont} \log (2 \pi \sigma^{2})}{2} - \frac{\Vert x_{i}^\textup{cont} - \mu_{k} \Vert_{2}^{2}}{2 \sigma^{2}} +
     \sum_{j=1}^{d_{\textup{cate}}}\sum_{c=1}^{l_j} \mathbb{I} (x_{ij}^\textup{cate}=c) \cdot \log p_{k,j,c}
    \right).
\end{equation}

\paragraph{Experimental setup}
For baseline methods, we consider SFC and VFC using continuous data $x_{i}^{cont}, i \in [N].$
We consider Adult dataset, which includes a binary label indicating whether an individual's annual income exceeds \$10K (the labels are not used during clustering). We set $K=2$ to match this binary label.
After clustering, we predict the label of each sample by its assigned cluster index.
Since cluster indices are assigned arbitrarily, we compute classification accuracy by evaluating both possible mappings between cluster indices and the ground-truth labels and selecting the mapping that yields the highest accuracy.
Note that a higher accuracy implies a better clustering quality.

\paragraph{Results}
\cref{table:cont+cate} below compares 
(i) continuous-only methods (SFC, VFC, and FMC applied to $x^\textup{cont}$ only)
and
(ii) FMC applied jointly to $(x^\textup{cont}, x^\textup{cate})$.
The results show that FMC with the mixed (continuous + categorical) features $(x^\textup{cont}, x^\textup{cate})$ outperforms SFC, VFC, and FMC with $x^\textup{cont}$ only.
In other words, FMC with the mixed features achieves a higher probability of label homogeneity within clusters, indicating that jointly using continuous and categorical data yields superior clustering performance, than using categorical data only.
This result demonstrates a clear advantage of our model-based approach over baseline methods, which can only handle continuous data.

\begin{table}[h]
  \centering
  \footnotesize
  \begin{tabular}{l|ccc}
    \toprule
    Method & Accuracy $\uparrow$ & $\Delta$ $\downarrow$ & Balance $\uparrow$
    \\
    \midrule
    SFC & 0.573 & 0.006 & 0.490
    \\
    FMC (continuous only, $\lambda = 10$) & 0.579 & 0.000 & 0.491
    \\
    FMC (continuous + categorical, $\lambda = 10$) & \textbf{0.706} & 0.000 & 0.488
    \\
    \arrayrulecolor{gray!80}\midrule\arrayrulecolor{black}
    VFC & 0.621 & 0.070 & 0.429
    \\
    FMC (continuous only, $\lambda = 1$) & 0.627 & 0.071 & 0.411
    \\
    FMC (continuous + categorical, $\lambda = 1$) & \textbf{0.703} & 0.054 & 0.430
    \\
    \bottomrule
  \end{tabular}
  \caption{
  Comparison of classification accuracy between continuous-only methods and FMC applied jointly to $(x^\textup{cont}, x^\textup{cate})$ for $K=2,$ on Adult dataset.
  We set $\lambda = 10$ when comparing with SFC ($\Delta \approx 0$) and $\lambda = 1$ when comparing with VFC ($\Delta \approx 0.07$).
  }
  \label{table:cont+cate}
\end{table}

\end{document}